\documentclass[twoside]{article}

\usepackage[accepted]{aistats2026}
\usepackage[round]{natbib}

\usepackage{xcolor}         % colors
\usepackage{soul}
\definecolor{lightgrey}{RGB}{192, 192, 192}
\definecolor{lightred}{rgb}{1, 0.8, 0.8}
\definecolor{lightblue}{rgb}{0.8, 0.9, 1}
\newcommand*\colourcheck[1]{%
  \expandafter\newcommand\csname #1check\endcsname{\textcolor{#1}{\ding{52}}}%
}
\colourcheck{blue}
\colourcheck{green}
\colourcheck{red}
\usepackage{pifont}

\usepackage[most]{tcolorbox}
\usepackage{colortbl}

\usepackage{hyperref}
\usepackage{url}
\usepackage{booktabs}
\usepackage{multirow}
\usepackage{multicol}
\usepackage{makecell}
\usepackage{amssymb}

\usepackage{xcolor}

\usepackage{color}
\usepackage{colortbl}

\usepackage{amsthm}
\theoremstyle{definition}

\newtheorem{theorem}{Theorem}
\newtheorem{lemma}{Lemma}

\usepackage{fontawesome5}
\usepackage{amsmath}
\usepackage{mathtools}
\usepackage{amsthm}

\usepackage{enumitem}

\usepackage{algorithm}
\usepackage{algpseudocode}

\usepackage{titletoc}
\usepackage{lipsum}
\newcommand{\GAM}{\textsc{GAM}}
\newcommand{\GAAM}{\textsc{GA\textsuperscript{2}M}}
\newcommand{\NAE}{\textsc{NAE}}

\begin{document}



\twocolumn[

\aistatstitle{Neural Additive Experts: Context-Gated Experts for Controllable Model Additivity}

\aistatsauthor{ Guangzhi Xiong \And Sanchit Sinha \And  Aidong Zhang }

\aistatsaddress{ University of Virginia \And  University of Virginia \And University of Virginia } ]
% \aistatsaddress{ University of Virginia} ]

\begin{abstract}
The trade-off between interpretability and accuracy remains a core challenge in machine learning. Standard Generalized Additive Models (GAMs) offer clear feature attributions but are often constrained by their strictly additive nature, which can limit predictive performance. Introducing feature interactions can boost accuracy yet may obscure individual feature contributions. To address these issues, we propose Neural Additive Experts (NAEs), a novel framework that seamlessly balances interpretability and accuracy. NAEs employ a mixture of experts framework, learning multiple specialized networks per feature, while a dynamic gating mechanism integrates information across features, thereby relaxing rigid additive constraints. Furthermore, we propose targeted regularization techniques to mitigate variance among expert predictions, facilitating a smooth transition from an exclusively additive model to one that captures intricate feature interactions while maintaining clarity in feature attributions. Our theoretical analysis and experiments on synthetic data illustrate the model's flexibility, and extensive evaluations on real-world datasets confirm that NAEs achieve an optimal balance between predictive accuracy and transparent, feature-level explanations. The code is available at \url{https://github.com/Teddy-XiongGZ/NAE}.
\end{abstract}

\section{INTRODUCTION}

The tension between interpretability and predictive accuracy is a central challenge in modern machine learning. While deep neural networks and other complex models have achieved remarkable success across a range of domains, their opaque decision-making processes often hinder adoption in high-stakes settings where transparency and trust are paramount. In contrast, Generalized Additive Models (GAMs) \citep{hastie2017generalized,agarwal2021neural,lou2012intelligible,chang2022node} are widely valued for their interpretability, the degree to which an observer can understand the cause of a decision \citep{biran2017explanation,miller2019explanation}, as they decompose predictions into clear, feature-specific contributions \citep{mcintosh2025generalized,ibrahim2025predicting,zhang2024gaussian,bouchiat2023improving}. However, the strictly additive structure of GAMs can limit their ability to capture complex relationships among features, often resulting in suboptimal predictive performance.

A common approach to improve the predictive accuracy of GAMs is to incorporate feature interactions \citep{ruppert2004elements,lou2013accurate}. However, introducing such interactions often diminishes interpretability, as it becomes more difficult to transparently attribute predictions to individual features, since interpreting a single feature may require a comprehensive analysis of multiple related feature interactions \citep{lou2012intelligible,caruana2015intelligible,christoph2020interpretable}. Furthermore, GAMs with interactions typically lack explicit mechanisms to control model additivity, the degree to which outputs can be decomposed into additive feature contributions, which is a hallmark of standard GAMs. This limitation restricts users' ability to flexibly balance accuracy and interpretability to suit specific application requirements.

In this work, we propose \emph{Neural Additive Experts} (NAEs), a novel framework designed to address this fundamental trade-off. NAEs extend the classical additive paradigm by associating each feature with a set of specialized expert networks. A dynamic gating mechanism adaptively integrates information across features, allowing the model to relax the rigid additive constraints when interactions are important. To control the balance between accuracy and interpretability, we introduce a targeted regularization term that controls the variance among expert predictions, enabling a smooth transition from a purely additive model to one that captures nuanced feature interactions while maintaining clarity in feature attributions.

Our theoretical analysis demonstrates how NAEs can recover complex data-generating processes that are inaccessible to standard GAMs, and provides insight into the role of regularization in balancing model flexibility and interpretability. Through experiments on synthetic and real-world datasets, we show that NAEs achieve competitive accuracy with state-of-the-art black-box models, while offering transparent, feature-level explanations. By providing a principled mechanism to control the trade-off between interpretability and accuracy, NAEs offer a practical solution for applications where both predictive performance and transparency are essential.

The contributions of this work are summarized as follows:
\begin{itemize}
\item We introduce \textbf{Neural Additive Experts} (NAEs), a novel framework that extends additive models with a mixture-of-experts architecture and dynamic gating, enabling flexible integration of feature interactions while preserving interpretability.
\item We propose targeted regularization techniques that allow users to control the variance among expert predictions, providing a principled mechanism to balance model flexibility and transparency.
\item We provide theoretical analysis demonstrating how NAEs can recover complex data-generating processes that are inaccessible to standard GAMs, and clarify the role of regularization in managing the trade-off between accuracy and interpretability.
\item We empirically validate NAEs on both synthetic and real-world datasets, showing that our approach achieves competitive predictive performance with state-of-the-art black-box models while maintaining clear, feature-level explanations.
\end{itemize}
\section{NEURAL ADDITIVE EXPERTS}

Neural Additive Expert (NAE) is a framework that extends classical GAMs by associating each feature with a set of expert predictors and employing a dynamic, context-aware gating mechanism. This architecture enables the model to capture complex, context-dependent feature effects while preserving the additive structure that supports interpretability. Figure \ref{fig:architecture} provides an overview of the NAE architecture.

\begin{figure*}[h!]
    \centering
    \includegraphics[width=0.9\linewidth]{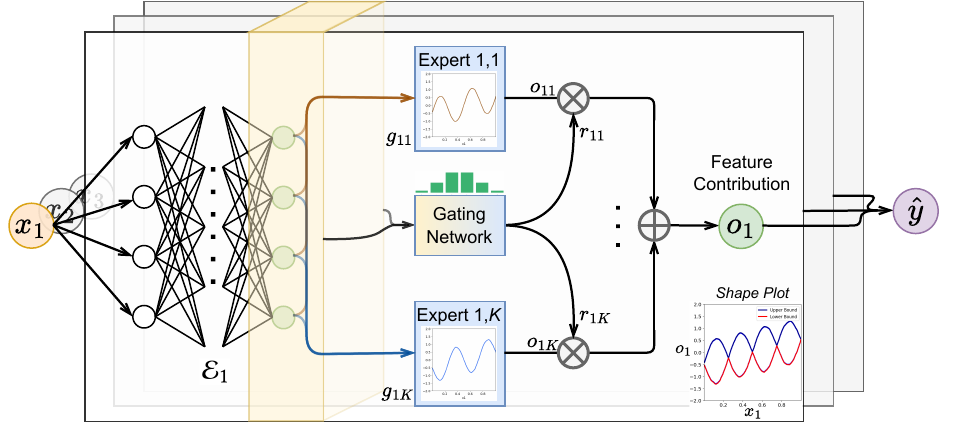}
    \caption{
    Illustration of the Neural Additive Expert (NAE) framework. 
    A gating network dynamically assigns relevance scores to multiple expert predictors for each feature, and the aggregated feature contributions are summed to produce the final prediction. This design maintains interpretability while enabling flexible modeling of complex relationships.
    }
    \label{fig:architecture}
\end{figure*}

\subsection{Feature Representation and Expert Networks}

Let $x_i$ ($i=1,\ldots,n$) denote the $i$-th feature of an input instance. Each feature is first mapped into a latent representation via an encoder:
\begin{equation}
    \mathcal{E}_i: \mathcal{X}_i \rightarrow \mathbb{R}^d,
\end{equation}
where $\mathcal{X}_i$ is the domain of $x_i$ and $d$ is the latent dimension. Unlike standard GAMs, which use a single predictor per feature, NAEs introduce a set of $K$ expert networks for each feature. The $k$-th expert for feature $x_i$ produces:
\begin{equation} \label{eq:expert_pred}
    o_{ik} = g_{ik}\bigl(\mathcal{E}_i(x_i)\bigr),
\end{equation}
where $g_{ik}: \mathbb{R}^d \rightarrow \mathbb{R}$ is typically a simple function (e.g., linear or shallow neural network) that maps the latent representation to a feature effect.

\subsection{Dynamic Gating and Expert Aggregation} \label{sec:gating}

A central component of NAEs is the dynamic gating mechanism, which adaptively integrates information across features to determine the relevance of each expert. For each feature $x_j$, we compute a score vector to quantify the influence of the \(x_i\) value on the expert selection for feature \(x_j\):
\begin{equation} \label{eq:aggregation}
    \varphi_j = \mu_j + \sum_{i=1}^{n} \mathcal{A}_{ij}^\top \mathcal{E}_i(x_i),
\end{equation}
where $\mathcal{A}_{ij} \in \mathbb{R}^{d \times K}$ are learnable parameters and $\mu_j \in \mathbb{R}^K$ is a learnable bias term. This formulation allows the gating to depend on the context provided by all features, relaxing the strict additivity of classical GAMs.
To ensure that only the most relevant experts contribute to the final prediction, a sparsity-inducing mask $M_j \in \mathbb{R}^K$ is applied to $\varphi_j$, where entries corresponding to less significant scores are set to $-\infty$, with remaining values set to be 0. This masking procedure effectively zeroes out the contribution of less important experts after normalization. The relevance weights are then obtained via a softmax operation:
\begin{equation} \label{eq:relevance}
    r_{jk} = \frac{\exp\bigl(\varphi_j[k] + M_j[k]\bigr)}{\sum_{l=1}^{K} \exp\bigl(\varphi_j[l] + M_j[l]\bigr)}, \quad k = 1,\cdots,K,
\end{equation}
which ensures that $r_{jk} \geq 0$ and $\sum_{k=1}^{K} r_{jk} = 1$. Subsequently, the output for feature $x_i$ is computed as a weighted aggregation of the individual expert predictions:
\begin{equation} \label{eq:expert_aggregation}
    o_i = \sum_{k=1}^{K} r_{ik} o_{ik},
\end{equation}
where \(o_{ik}\) is the output of the \(k\)-th expert for feature \(x_i\) as defined in Equation \ref{eq:expert_pred}. The overall model prediction is obtained by summing the feature-specific outputs along with an intercept term \(\omega_0\):
\begin{equation} \label{eq:final_prediction}
    \hat{y} = \omega_0 + \sum_{i=1}^{n} o_i.
\end{equation}
This structure preserves feature-level interpretability while allowing for adaptive, context-dependent modeling. The computational complexity of NAEs is discussed in Appendix \ref{sec:complexity}.

\subsection{Training Objective and Regularization}

Given a dataset $\{(x^t, y^t)\}_{t=1}^N$, NAEs are trained to optimize:
\begin{equation}\label{eq:loss}\small
\begin{aligned}
    \text{minimize} \quad &\frac{1}{N}\sum_{t=1}^N \mathcal{L}\bigl(y^t, \hat{y}^t\bigr) + \\
     &\frac{\lambda}{nNK}\sum_{t=1}^N\sum_{i=1}^n\sum_{k=1}^K\left(o_{ik}^t- \frac{1}{K} \sum_{l=1}^{K} o_{il}^t \right)^2,
\end{aligned}
\end{equation}
where $\mathcal{L}$ is a task-specific loss (e.g., mean squared error or cross-entropy), and $\lambda$ controls the strength of the expert variation penalty. This regularization encourages consistency among experts for each feature, providing a tunable mechanism to balance flexibility and interpretability. Standard regularization techniques (e.g., dropout, L2 regularization) are also applied to promote generalization.

\subsection{Interpretability and Feature Attribution}

NAEs maintain interpretability by quantifying the contribution of each feature to the final prediction. For each feature $x_i$, the model aggregates expert outputs using the learned gating weights, ensuring that feature attributions remain bounded and meaningful. To further aid interpretation, we define upper and lower bounds for each feature's effect:
\begin{equation}\small
\begin{aligned}
    \text{upper}(o_i|x_i) &= \max_{k} g_{ik}\bigl(\mathcal{E}_i(x_i)\bigr), \\
    \text{lower}(o_i|x_i) &= \min_{k} g_{ik}\bigl(\mathcal{E}_i(x_i)\bigr).
\end{aligned}
\end{equation}
These bounds characterize the range of possible contributions for each feature, and visualizing the distribution of actual feature effects within this range provides insight into both the magnitude and variability of feature influences.

While feature attribution is central to NAEs, the model also enables analysis of feature interactions. Since the score vector in Equation~\eqref{eq:aggregation} is constructed as an additive function over all features, we can isolate the interaction between features \(x_i\) and \(x_j\) by focusing on their respective contributions in the gating function:
\begin{equation}
    \varphi_j' = \mathcal{A}_{ij}^\top \mathcal{E}_i(x_i) \qquad\text{and}\qquad \varphi_i' = \mathcal{A}_{ji}^\top \mathcal{E}_j(x_j).
\end{equation}
By removing the influence of other features, the updated feature outputs $o_i'$ and $o_j'$ with Equation \eqref{eq:expert_aggregation} reflect the pairwise feature interactions captured by NAE. Further discussion of feature interactions is provided in Appendix \ref{sec:ga2m}.
\section{THEORETICAL ANALYSIS OF NAES}
\label{sec:theory}

In this section, we characterize the expressivity of NAEs relative to classical additive families and analyze how they control model additivity for feature-level interpretability. Throughout we assume each feature domain $\mathcal X_i\subset\mathbb R$ is compact and all functions are continuous on their domains. NAEs use (i) per-feature encoders $E_i:\mathcal X_i\to\mathbb R^d$, (ii) expert functions $g_{ik}:\mathbb R^d\to\mathbb R$, (iii) softmax gating weights $r_{jk}(x)\ge 0$ with $\sum_{k=1}^K r_{jk}(x)=1$, and (iv) per-feature outputs,
\begin{equation*}\small
\begin{aligned}
o_i(x) = \sum_{k=1}^K r_{ik}(x)  g_{ik}\!\big(E_i(x_i)\big), \quad \hat y(x) = \omega_0+\sum_{i=1}^n o_i(x),
\end{aligned}
\end{equation*}
as in Equations \eqref{eq:expert_pred}-\eqref{eq:final_prediction}. Training uses the loss and expert-variation penalty from Equation \eqref{eq:loss}. 
For theoretical clarity, we assume a dense expert routing architecture in which all $K$ experts are available for each feature and gating is performed over the full set.

\paragraph{Model classes.}
Let $\GAM$ denote generalized additive models
\begin{equation}\small
\GAM =\Big\{ f(x)=\omega_0+\sum_{i=1}^n f_i(x_i) \Big| f_i \ \text{continuous}\Big\}.
\end{equation}
Let $\GAAM$ denote generalized additive models with pairwise interactions (GA$^2$M):
\begin{equation}\small
\begin{aligned}
\GAAM =\Big\{&f(x)=\omega_0+\sum_{i=1}^n f_i(x_i)+\\
&\sum_{1\le i<j\le n} f_{ij}(x_i,x_j)\Big| f_i,f_{ij}\ \text{continuous}\Big\}.
\end{aligned}
\end{equation}
We write $\NAE(K)$ for the set of functions represented by an NAE with $K$ experts per feature.

\subsection{Exact containment of GAMs}
\begin{theorem}[GAM containment]
\label{thm:GAM}
For any $f\in\GAM$ there exists $K=1$ and parameters of an NAE such that $\hat y(x)=f(x)$ for all $x$.
\end{theorem}
\begin{proof}
Choose $K=1$, set $r_{i1}(x)\equiv 1$, and choose $g_{i1}\!\circ E_i=f_i$. Then $o_i(x)=f_i(x_i)$ and $\hat y(x)=\omega_0+\sum_i f_i(x_i)=f(x)$.
\end{proof}

\subsection{From GA$^2$M to NAE via separable approximation}
To show the containment relative to GA$^2$M, we need the following approximation lemma, which follows from the Stone-Weierstrass theorem.

\begin{lemma}[Finite separable approximation of pairwise terms]
\label{lem:sep}
Fix $(i,j)$ and $\varepsilon>0$. There exist $M\in\mathbb N$ and continuous functions $\{u_m:\mathcal X_i\to\mathbb R,\ v_m:\mathcal X_j\to\mathbb R\}_{m=1}^M$ such that
\(
\sup_{(x_i,x_j)}\big|f_{ij}(x_i,x_j)-\sum_{m=1}^M u_m(x_i)v_m(x_j)\big|<\varepsilon.
\)
\end{lemma}

We now show that each separable term $u(x_i)v(x_j)$ is representable by a single feature head $o_i(\cdot)$ using two experts and a softmax gate that depends on $x_j$ only.

\begin{lemma}[Two-expert product construction]
\label{lem:two-expert}
Let $u:\mathcal X_i\!\to\!\mathbb R$ and $v:\mathcal X_j\!\to\!\mathbb R$ be bounded and continuous. Choose $C>\sup_{x_j}|v(x_j)|$. Instantiate two experts for feature $i$ with outputs
\(
g_{i,+}\!\circ E_i(x_i)=+C u(x_i),\quad
g_{i,-}\!\circ E_i(x_i)=-C u(x_i),
\)
and gate them by logits $\varphi_i^{(+)}(x)=\alpha(x)-\beta(x_j)$, $\varphi_i^{(-)}(x)=\alpha(x)+\beta(x_j)$ with softmax over $\{+,-\}$ only. Then
\begin{equation}\small
r_{i,+}(x)-r_{i,-}(x) = - \tanh\big(\beta(x_j)\big).
\end{equation} 
Setting $\beta(x_j)= - \operatorname{arctanh}\!\big(v(x_j)/C\big)$ yields
\begin{equation}\small
o_i(x)=r_{i,+}(x) C u(x_i) - r_{i,-}(x) C u(x_i)
=u(x_i) v(x_j).
\end{equation}
Moreover, we can enforce the ``pairwise-only'' dependence by choosing the gating parameters so that $\varphi_i^{(\pm)}$ depend on $x_j$ but not on other coordinates (i.e., zeroing $A_{\ell i}$ for $\ell\not=j$ in Equation \eqref{eq:aggregation}).
\end{lemma}

\begin{proof}[Proof sketch]
With softmax, if the logits are $\alpha\pm\beta$, the probabilities are $\sigma(-2\beta)$ and $\sigma(2\beta)$, hence their difference is $- \tanh(\beta)$. The rest is algebra using boundedness of $v$ and the choice of $C$ to keep $r_{i,\pm}\in[0,1]$.
\end{proof}

\begin{theorem}[GA$^2$M containment up to arbitrary precision]
\label{thm:ga2m}
For any $f\in\GAAM$ and any $\epsilon>0$ there exists $K$ and an NAE such that $\sup_x |\hat y(x)-f(x)|<\epsilon$. In particular, for each feature $i$ it suffices to take
\begin{equation}
K_i \le1+2\!\!\sum_{j\ne i}\!M_{ij},
\end{equation}
where $M_{ij}$ is the number of separable terms used for $f_{ij}$ in Lemma \ref{lem:sep}.
\end{theorem}
\begin{proof}[Proof sketch]
Write
$
f(x)=\omega_0+\sum_i f_i(x_i)+\sum_{i<j} \sum_{m=1}^{M_{ij}} u_{ijm}(x_i)v_{ijm}(x_j)+r(x),
$
with $\|r\|_\infty<\epsilon/2$ by Lemma \ref{lem:sep} on each pair. Realize each $f_i$ with a dedicated single expert (as in Theorem \ref{thm:GAM}). For each separable $u_{ijm}v_{ijm}$ assign it to feature head $i$ and apply Lemma \ref{lem:two-expert}. Summing all heads and the intercept recovers $f$ up to the residual $r$, giving uniform error $<\epsilon$. The bound on $K_i$ counts one additive expert plus two per separable term.
\end{proof}

Thus, $\GAM \subsetneq \overline{\NAE(K)}$ (closure in uniform norm, i.e. functions uniformly approximable by $\NAE$) whenever gates are allowed to depend on other features. Hence NAEs are more expressive than vanilla GAM/GA$^2$M but remain feature-decomposed at prediction time.

\subsection{Controlling additivity via the expert-variation penalty}

A key feature of NAEs is the ability to interpolate between strictly additive and highly flexible models by tuning the expert variation penalty. To quantify this, we define an additivity metric:
\begin{equation}\label{eq:additivity}
\text{Additivity} = \frac{1}{n}\sum_{i=1}^n \frac{\text{Var}\bigl(\mathbb{E}(o_i|x_i)\bigr) + \delta}{\text{Var}(o_i) + \delta},
\end{equation}
where $o_i$ is the contribution of feature $i$ to the prediction, and $\delta$ is a small constant for numerical stability. For a strictly additive model, this metric equals 1.
The training objective of NAEs includes an expert variation penalty weighted by $\lambda$ (Equation \eqref{eq:loss}), which encourages consistency among experts for each feature. As $\lambda$ increases, the model becomes more additive:

\begin{theorem}[Monotone additivity and additive limit]
\label{thm:additivity}
Let $A(\lambda)$ be the additivity metric of an NAE trained with penalty $\lambda\ge 0$. Then $A(\lambda)$ is nondecreasing in $\lambda$. Moreover, $\lim_{\lambda\to\infty} A(\lambda)=1$, and any limit point of minimizers as $\lambda\to\infty$ is a GAM, realized by setting $r_{ik}\equiv$ constants and $g_{ik}$ equal across $k$ for each feature.
\end{theorem}
\begin{proof}[Proof sketch]
The penalty forces $g_{ik}(E_i(x_i))$ to contract toward their within-feature mean, collapsing each head to a single expert in the limit; gates then become irrelevant and the model reduces to an additive form, for which the metric equals $1$. Monotonicity follows from standard convex-order arguments on the within-feature variance term. Full proofs are given in the supplement.
\end{proof}

\section{EXPERIMENTS ON SIMULATED DATA}\label{sec:empirical_simul}

In this section, we empirically evaluate NAEs on synthetic datasets specifically constructed to assess their ability to model both additive and non-additive (multimodal) data distributions. We benchmark NAEs against traditional GAMs, implemented as Neural Additive Models (NAMs) \citep{agarwal2021neural}. Our experiments address two main questions: (1) Can NAEs accurately recover underlying shape functions in both unimodal (additive) and multimodal (non-additive) scenarios? (2) How does the expert variation penalty parameter \(\lambda\) influence the trade-off between model flexibility and additivity? Additional simulation studies on robustness to feature sparsity and increasing distributional complexity are provided in Appendix \ref{sec:simu_more}.

\subsection{Simulation Setup and Visualization}

We generate two synthetic datasets, each with 10{,}000 samples:

\noindent\textbf{Unimodal Distribution (Additive).} With \(x_1 \sim U(0,1)\) and \(\sigma = 0.1\), the response is sampled as:
\begin{equation}
    y \sim \mathcal{N}\Big(x_1 -\frac{1}{2} + \sin(4\pi x_1), \sigma^2 \Big),
\end{equation}

\noindent\textbf{Multimodal Distribution (Non-Additive).} With \(x_2\) uniformly sampled from \(\{-1,1\}\), the response is generated following:
\begin{equation}\label{eq:multimodal}
y \sim
\begin{cases}
    \mathcal{N}\Big(x_1 - \frac{1}{2} - \sin(4\pi x_1), \sigma^2 \Big), & \text{if } x_2 = -1,\\
    \mathcal{N}\Big(x_1 - \frac{1}{2} + \sin(4\pi x_1), \sigma^2 \Big), & \text{if } x_2 = 1,
\end{cases}
\end{equation}

Figure \ref{fig:simul_shape} compares the learned shape functions of NAM and NAE. In the unimodal setting, both models accurately recover the underlying pattern, and NAE maintains interpretability without introducing spurious complexity. In the multimodal setting, NAM produces a near-linear fit, failing to capture the interaction, while NAE recovers the multimodal structure, with upper and lower expert bounds reflecting the range of feature contributions.

\begin{figure}[h!]
    \centering
    \includegraphics[width=0.85\linewidth]{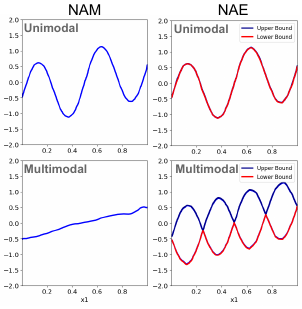}
    \caption{Shape functions learned by NAM and NAE on simulated data. For the multimodal case, NAE captures the oscillatory structure, while NAM fails.}
    \label{fig:simul_shape}
\end{figure}

\begin{table*}[h!]
    \centering
    \caption{Performance comparison on benchmark datasets. 
    ``Complex.'' denotes the explanation complexity, i.e., the number of components required to fully explain model predictions, with \(n\) representing the number of features.
    The best-performing models in each category are highlighted.
    }
\begin{center}
    \resizebox{\linewidth}{!}{
    \begin{tabular}{lccccccccccc}
        \toprule
        \bf \multirow{2.5}{*}{Model} & \bf \multirow{2.5}{*}{\makecell{Interpret-\\ability}} & \bf \multirow{2.5}{*}{\makecell{Complex.}} & \bf Housing & \bf \makecell{MIMIC-II} & \bf MIMIC-III & \bf Income & \bf Credit & \bf Year \\
        \cmidrule(lr){4-4}
        \cmidrule(lr){5-5}
        \cmidrule(lr){6-6}
        \cmidrule(lr){7-7}
        \cmidrule(lr){8-8}
        \cmidrule(lr){9-9}
         & & & RMSE $\downarrow$ & AUC $\uparrow$ & AUC $\uparrow$ & AUC $\uparrow$ & AUC $\uparrow$ & MSE $\downarrow$ \\
        \midrule
        \rowcolor[RGB]{234, 238, 234} \multicolumn{9}{c}{\textbf{Black-box Models}} \\
        \midrule
        MLP & No & N/A & \makecell{0.501 \scriptsize$\pm$0.006} & \makecell{0.835 \scriptsize$\pm$0.014} & \makecell{0.815 \scriptsize$\pm$0.009} & \makecell{0.914 \scriptsize$\pm$0.003} & \bf \makecell{0.981 \scriptsize$\pm$0.007} & \makecell{78.48 \scriptsize$\pm$0.56} \\
        NODE & No & N/A & \makecell{0.523 \scriptsize$\pm$0.000} & \makecell{0.843 \scriptsize$\pm$0.011} & \bf \makecell{0.828 \scriptsize$\pm$0.007} & \makecell{0.919 \scriptsize$\pm$0.003} & \makecell{0.981 \scriptsize$\pm$0.009} & \bf \makecell{76.21 \scriptsize$\pm$0.12} \\
        XGBoost & No & N/A & \bf \makecell{0.443 \scriptsize$\pm$0.000} & \bf \makecell{0.844 \scriptsize$\pm$0.012} & \makecell{0.819 \scriptsize$\pm$0.004} & \bf \makecell{0.928 \scriptsize$\pm$0.003} & \makecell{0.978 \scriptsize$\pm$0.009} & \makecell{78.53 \scriptsize$\pm$0.09} \\
        \midrule
        \rowcolor[RGB]{234, 238, 234} \multicolumn{9}{c}{\textbf{Interaction-explaining Models}} \\
        \midrule
        EB$^2$M & Yes & $O(n^2)$ & \makecell{0.492 \scriptsize$\pm$0.000} & \bf \makecell{0.848 \scriptsize$\pm$0.012} & \makecell{0.821 \scriptsize$\pm$0.004} & \bf \makecell{{0.928} \scriptsize$\pm$0.003} & \makecell{{0.982} \scriptsize$\pm$0.006} & \makecell{83.16 \scriptsize$\pm$0.01} \\
        NA$^2$M & Yes & $O(n^2)$ & \makecell{0.492 \scriptsize$\pm$0.008} & \makecell{0.843 \scriptsize$\pm$0.012} & \bf \makecell{{0.825} \scriptsize$\pm$0.006} & \makecell{0.912 \scriptsize$\pm$0.003} & \makecell{{0.985} \scriptsize$\pm$0.007} & \makecell{79.80 \scriptsize$\pm$0.05} \\
        NB$^2$M & Yes & $O(n^2)$ & \makecell{0.478\scriptsize$\pm$0.002} & \bf \makecell{{0.848} \scriptsize$\pm$0.012} & \makecell{0.819 \scriptsize$\pm$0.010} & \makecell{0.917 \scriptsize$\pm$0.003} & \makecell{0.978 \scriptsize$\pm$0.007} & \makecell{79.01 \scriptsize$\pm$0.03} \\        
        NODE-GA$^2$M & Yes & $O(n^2)$ & \bf \makecell{0.476 \scriptsize$\pm$0.007} & \makecell{{0.846} \scriptsize$\pm$0.011} & \makecell{{0.822} \scriptsize$\pm$0.007} & \makecell{0.923 \scriptsize$\pm$0.003} & \bf \makecell{{0.986} \scriptsize$\pm$0.010} & \bf \makecell{79.57 \scriptsize$\pm$0.12} \\
        \midrule
        \rowcolor[RGB]{204, 255, 255} \multicolumn{9}{c}{\textbf{Feature-explaining Models}} \\
        \midrule
        Linear & Yes & $O(n)$ & \makecell{0.735 \scriptsize$\pm$0.000} & \makecell{0.796 \scriptsize$\pm$0.012} & \makecell{0.772 \scriptsize$\pm$0.009} & \makecell{0.900 \scriptsize$\pm$0.002} & \makecell{0.976 \scriptsize$\pm$0.012} & \makecell{88.89 \scriptsize$\pm$0.40}\\
        Spline & Yes & $O(n)$ & \makecell{0.568 \scriptsize$\pm$0.000} & \makecell{0.825 \scriptsize$\pm$0.011} & \makecell{0.812 \scriptsize$\pm$0.004} & \makecell{0.918 \scriptsize$\pm$0.003} & \makecell{0.982 \scriptsize$\pm$0.011} & \makecell{85.96 \scriptsize$\pm$0.07} \\
        EBM & Yes & $O(n)$ & \makecell{0.559 \scriptsize$\pm$0.000} & \makecell{0.835 \scriptsize$\pm$0.011} & \makecell{0.809 \scriptsize$\pm$0.004} & \bf \makecell{{0.927} \scriptsize$\pm$0.003} & \makecell{0.974 \scriptsize$\pm$0.009} & \makecell{85.81 \scriptsize$\pm$0.11} \\
        NAM & Yes & $O(n)$ & \makecell{0.572 \scriptsize$\pm$0.005} &  \makecell{0.834 \scriptsize$\pm$0.013} & \makecell{0.813 \scriptsize$\pm$0.003} & \makecell{0.910 \scriptsize$\pm$0.003} & \makecell{0.977 \scriptsize$\pm$0.015} & \makecell{85.25 \scriptsize$\pm$0.01} \\
        NBM & Yes & $O(n)$ & \makecell{0.564 \scriptsize$\pm$0.001} & \makecell{0.833 \scriptsize$\pm$0.013} & \makecell{0.806 \scriptsize$\pm$0.003} & \makecell{0.918 \scriptsize$\pm$0.003} & \makecell{0.981 \scriptsize$\pm$0.007} & \makecell{85.10 \scriptsize$\pm$0.01} \\
        NODE-GAM & Yes & $O(n)$ & \makecell{0.558 \scriptsize$\pm$0.003} & \makecell{0.832 \scriptsize$\pm$0.011} & \makecell{0.814 \scriptsize$\pm$0.005} & \bf \makecell{{0.927} \scriptsize$\pm$0.003} & \makecell{0.981 \scriptsize$\pm$0.011} & \makecell{85.09 \scriptsize$\pm$0.01} \\
        \bf  NAE & Yes & $O(n)$ &  \bf \makecell{0.451 \scriptsize$\pm$0.002} & \bf \makecell{{0.847} \scriptsize$\pm$0.014} & \bf \makecell{{0.825} \scriptsize$\pm$0.006} & \bf \makecell{{0.927} \scriptsize$\pm$0.003} & \bf \makecell{{0.982} \scriptsize$\pm$0.009} & \bf \makecell{{78.66} \scriptsize$\pm$0.21} \\
        \bottomrule
    \end{tabular}
    }
\end{center}
     \label{tab:accuracy}
\end{table*}

\subsection{Effect of Variation Penalty \(\lambda\)}

The expert variation penalty \(\lambda\) in NAEs controls the balance between flexibility and additivity, as shown in our theoretical analysis (Section \ref{sec:theory}). To further validate this, we train NAEs on the multimodal simulated dataset with different regularization strengths \(\lambda=0.1, 1, 10\) and visualize the learned shape functions for \(x_1\).

As shown in Figure \ref{fig:simul_lambda}, small \(\lambda\) values allow for high flexibility, which perfectly recovers the underlying distribution and has a wide gap between upper and lower bounds. As \(\lambda\) increases, the bounds converge, enforcing additivity of the trained model. The corresponding additivity scores (0.597, 0.709, and 1.000 for \(\lambda=0.1, 1, 10\)) confirm the theoretical predictions.

\begin{figure}[h!]
    \centering
    \includegraphics[width=1.0\linewidth]{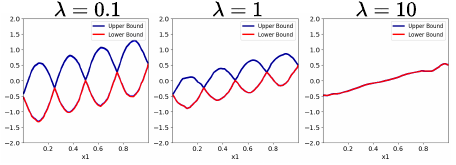}
    \caption{Effect of the expert variation penalty \(\lambda\) on the learned shape function of \(x_1\) in NAEs. As \(\lambda\) increases, the model transitions from flexible to strictly additive.}
    \label{fig:simul_lambda}
\end{figure}

\section{EXPERIMENTS ON REAL-WORLD DATA}\label{sec:empirical_real}

In this section, we evaluate NAEs on a suite of real-world datasets, focusing on both predictive accuracy and interpretability. We describe the datasets and baseline models, analyze predictive performance and feature attributions, and investigate the trade-off between accuracy and interpretability. Additional ablation studies on the gating mechanism and number of experts are provided in Appendices \ref{sec:mixnam-d}, \ref{sec:mixnam-e}, and \ref{sec:scaling}.

\subsection{Datasets and Baselines}

We evaluate NAEs on six widely used datasets spanning regression and classification tasks: Housing \citep{pace1997sparse}, MIMIC-II \citep{saeed2011multiparameter}, MIMIC-III \citep{johnson2016mimic}, Income \citep{blake1998uci}, Credit, and Year. These datasets vary in size, feature types, and the presence of categorical variables, providing a comprehensive testbed. Dataset statistics and more details are in Appendix \ref{sec:data_desc}.

\begin{figure*}[h!]
    \centering
    \includegraphics[width=0.95\linewidth]{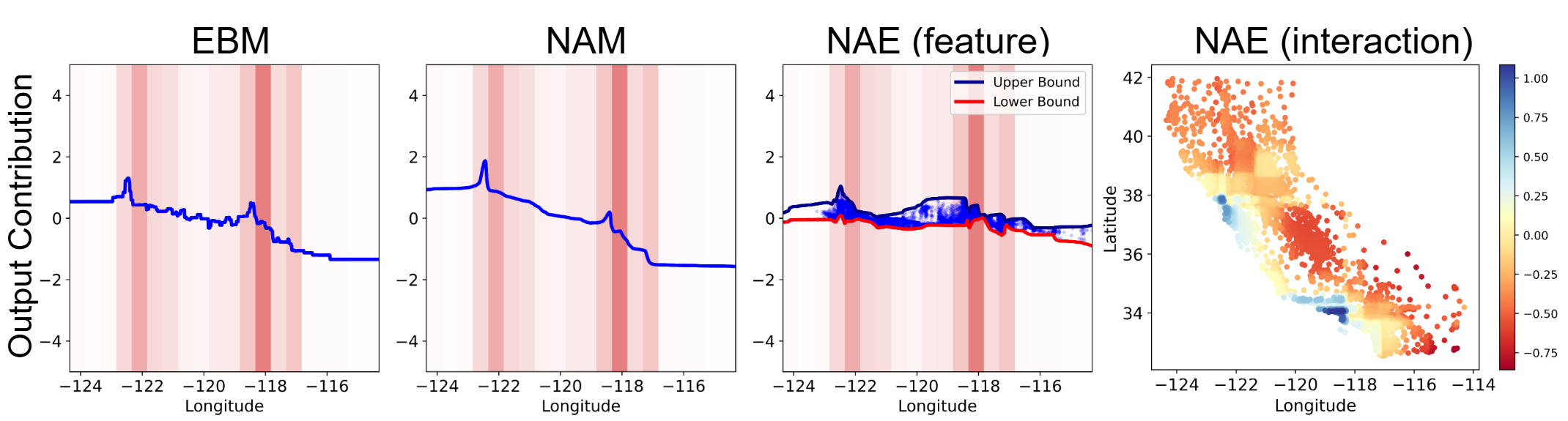}
    \caption{Comparison of Longitude effect on house price predictions across models. Y-axis shows mean-centered feature contributions. Background bars indicate normalized data density. Blue dots show actual feature effects. Longitude-Latitude interaction learned by NAE is also shown.}
    \label{fig:explanation}
\end{figure*}

We benchmark NAEs against a range of baselines: linear and spline models (basic interpretability), Neural Additive Models (NAM) \citep{agarwal2021neural}, Neural Basis Models (NBM) \citep{radenovic2022neural}, Explainable Boosting Machine (EBM) \citep{lou2012intelligible,nori2019interpretml}, NODE-GAM \citep{chang2022node}, and their pairwise-interaction variants (e.g., NA$^2$M, NB$^2$M, NODE-GA$^2$M). Black-box models such as MLP, NODE \citep{popov2020neural}, and XGBoost \citep{chen2016xgboost} are also included. Implementation details are in Appendices \ref{sec:model_implement} and \ref{sec:app_hyper}.

\subsection{Predictive Performance and Feature Attribution}

Table \ref{tab:accuracy} compares NAEs to all baselines. Metrics are reported with arrows indicating the desired direction. 
The results show that while traditional additive models (e.g., EBM, NAM) provide feature-level explanations, they often underperform compared to models that capture interactions or use black-box architectures. NAEs, however, match or exceed the predictive performance of these more complex models, while maintaining clear, feature-level interpretability.

The ``Complex.'' column in the table refers to the number of components required for a full model inspection. For interaction-level GA$^2$M models, the prediction sums univariate functions and pairwise interaction surfaces, meaning a complete inspection requires examining $O(n^2)$ components. In contrast, an NAE prediction always decomposes into exactly one scalar contribution per feature. The primary explanation thus consists of only $O(n)$ shape plots, each displaying the feature's main effect and its associated interaction strength bounds. 
In high-dimensional settings, this reduction from $O(n^2)$ to $O(n)$ explanation complexity offers a significant interpretability advantage.

Figure \ref{fig:explanation} demonstrates the interpretability of NAEs by comparing the effect of the Longitude feature on house price predictions across models. EBM, NAM, and NAE all capture similar geographic trends, with higher property values near key longitudes (e.g., -122.5 for San Francisco, -118.5 for Los Angeles). NAE further reveals a wide range of possible outcomes between -120 and -119, with most data points near the lower bound. Its Longitude-Latitude interaction plot confirms that coastal regions positively influence house prices, while other areas in the same longitude range may have a negative effect.

\subsection{Controlling the Trade-off Between Accuracy and Interpretability} \label{sec:balance}

As demonstrated by the theoretical analysis and experiments on the simulated data, NAEs provide a tunable trade-off between flexibility and additivity via the expert variation penalty parameter $\lambda$. Figure \ref{fig:housing_lambda} shows shape plots for the Longitude feature with different $\lambda$ values. As $\lambda$ increases, the bounds on feature effects tighten, enforcing additivity but potentially reducing flexibility.

\begin{figure}[h!]
\centering
\includegraphics[width=1.0\linewidth]{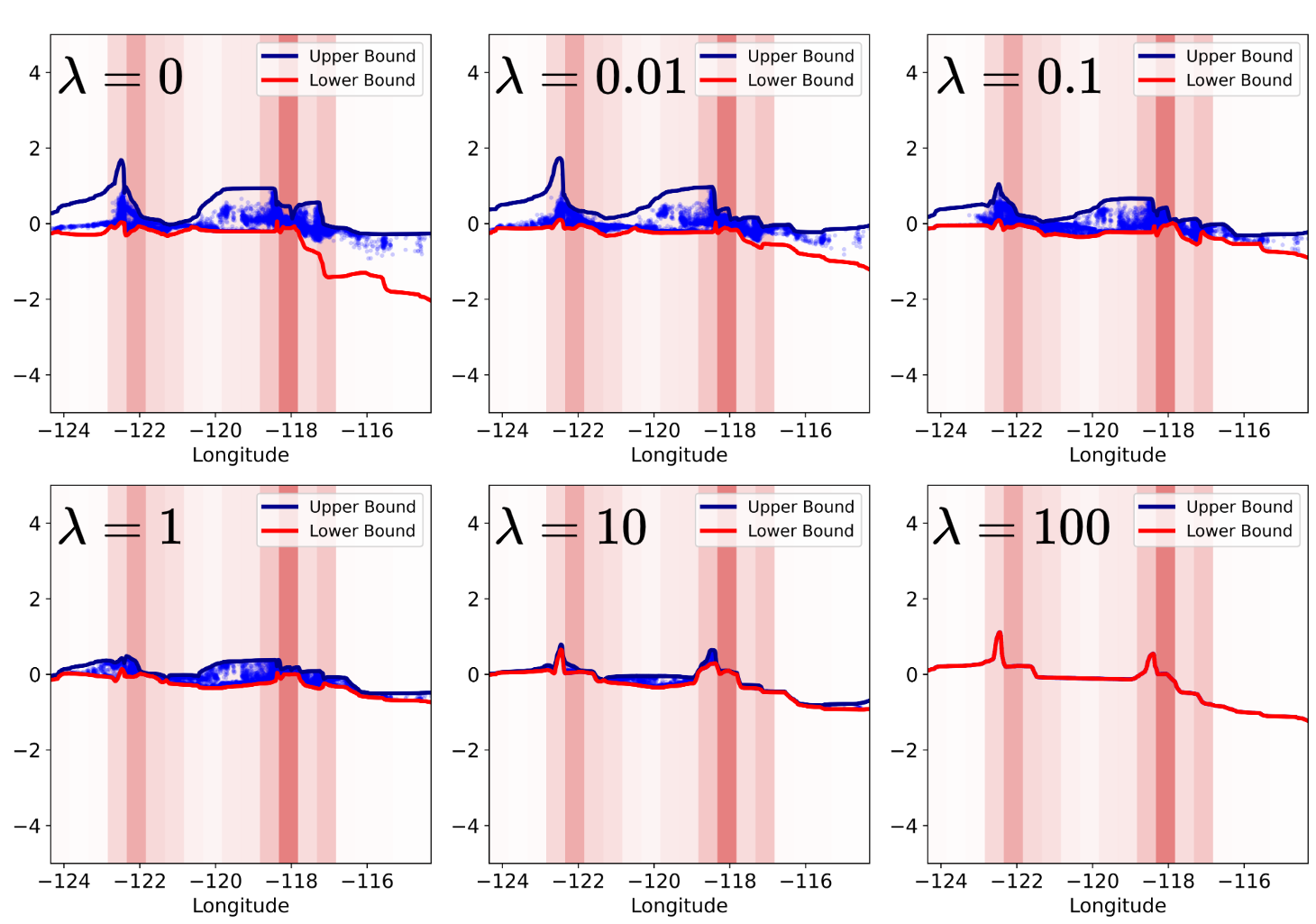}
\caption{Shape plots of the Longitude feature learned by NAEs with different $\lambda$ values. Higher $\lambda$ enforces additivity and narrows the range of possible feature effects.}
\label{fig:housing_lambda}
\end{figure}

To quantify this trade-off, we report both the additivity metric from Equation~\eqref{eq:additivity} and the RMSE of the trained NAEs. Additionally, we introduce a tightness metric to assess how closely the predicted bounds capture the estimated feature effects, reflecting the variability in the model's predictions. The tightness metric is defined as follows:
\begin{equation}
\text{Tightness} = \frac{1}{n}\sum_{i=1}^n \mathbb{E}\Big[\frac{\max(o_i|x_i) - \min(o_i|x_i) + \delta}{\text{upper}(o_i|x_i) - \text{lower}(o_i|x_i) + \delta}\Big].
\end{equation}
A high tightness value indicates that the predicted values closely span the estimated bounds, while a low value suggests the bounds are overly loose.

\begin{table}[h!]\scriptsize
\centering
\caption{Quantitative analysis of additivity and bound tightness on the Housing dataset for various \(\lambda\) settings, along with corresponding performance (RMSE).
}
\begin{center}
\begin{tabular}{cccccccccccc}
\toprule
\bf $\lambda=$ 0 & \bf $\lambda=$ 0.01 & \bf $\lambda=$ 0.1 & \bf $\lambda=$ 1 & \bf $\lambda=$ 10 & \bf $\lambda=$ 100 \\
\midrule
\multicolumn{6}{c}{Additivity} \\
\midrule
\makecell{0.522\\\scriptsize$\pm$0.000} & \makecell{0.537\\\scriptsize$\pm$0.000} & \makecell{0.562\\\scriptsize$\pm$0.012} & \makecell{0.639\\\scriptsize$\pm$0.000} & \makecell{0.897\\\scriptsize$\pm$0.000} & \makecell{1.000\\\scriptsize$\pm$0.000} \\
\midrule
\multicolumn{6}{c}{Tightness}\\
\midrule
\makecell{0.860\\\scriptsize$\pm$0.000} & \makecell{0.895\\\scriptsize$\pm$0.018} & \makecell{0.953\\\scriptsize$\pm$0.011} & \makecell{0.988\\\scriptsize$\pm$0.000} & \makecell{0.999\\\scriptsize$\pm$0.000} & \makecell{1.000\\\scriptsize$\pm$0.000}  \\      
\midrule
\multicolumn{6}{c}{RMSE} \\
\midrule
\makecell{0.451\\\scriptsize$\pm$0.003} & \makecell{0.450\\\scriptsize$\pm$0.002} & \makecell{0.451\\\scriptsize$\pm$0.003} & \makecell{0.458\\\scriptsize$\pm$0.003} & \makecell{0.515\\\scriptsize$\pm$0.008} & \makecell{0.582\\\scriptsize$\pm$0.008}  \\
\bottomrule
\end{tabular} \label{tab:additivity_tightness}
\end{center}
\end{table}

Table \ref{tab:additivity_tightness} summarizes additivity, tightness, and RMSE on the Housing dataset for different~$\lambda$ values. As~$\lambda$ increases, additivity and tightness improve, but predictive performance (RMSE) may degrade. This confirms that~$\lambda$ provides a practical mechanism for balancing accuracy and interpretability in NAEs. 
In our main experiments, we use $\lambda=0.1$ uniformly across datasets, as it offers a good balance between flexibility and additivity. 
For practical use, we suggest starting with $\lambda=0.1$ and increasing it if greater interpretability is needed, monitoring both validation accuracy and the additivity metric.
\section{RELATED WORK}

\subsection{Generalized Additive Models}

Generalized Additive Models (GAMs) are foundational in interpretable machine learning, modeling the response as a sum of univariate functions for clear, feature-level attributions \citep{hastie2017generalized}. This additive structure enables visualization and understanding of individual feature effects, making GAMs valuable in transparency-critical domains. Extensions include hierarchical modeling \citep{pedersen2019hierarchical}, robust evaluation \citep{hegselmann2020evaluation,chang2021interpretable}, and higher-order interactions \citep{enouen2022sparse,tan2018learning,duong2024cat}. Recent advances such as Explainable Boosting Machines (EBMs) \citep{lou2012intelligible}, Neural Additive Models (NAMs) \citep{agarwal2021neural}, NODE-GAM \citep{chang2022node}, and ProtoNAM \citep{xiong2025protonam} improve flexibility and accuracy while preserving interpretability. Further developments, including Neural Basis Models (NBMs) \citep{radenovic2022neural} and Gaussian Process Neural Additive Models (GP-NAMs) \citep{zhang2024gaussian}, enhance parameter efficiency, scalability, and uncertainty quantification \citep{bouchiat2023improving,thielmann2024neural}. However, the strictly additive nature of GAMs limits their ability to capture complex, context-dependent feature interactions \citep{cui2020factored,mueller2024gamformer,ibrahim2023grand}.

Previous work has attempted to address this limitation by introducing models such as GA$^2$M, where interactions are also captured \citep{lou2013accurate}. GA$^2$M primarily relies on 2D interaction plots for interpretability, which can become computationally expensive and challenging to visualize as the number of features increases. 
Also, multiple interaction plots are often required to fully understand the effect of a single feature in GA$^2$M, which complicates interpretation.
While post-hoc heterogeneity analysis methods were proposed to estimate the single feature effect in GA$^2$M without visualizing all interactions \citep{goldstein2015peeking,gkolemis2023rhale}, their explanations are empirical variability bands based on given samples. Although such bands can be informative, they do not guarantee coverage for unseen samples. In contrast, NAEs provide architectural (model-intrinsic) bounds on feature effects defined by the minimum and maximum outputs of the learned expert functions, ensuring coverage for any context and input, including those not seen during training.

Beyond GAMs, there are also post-hoc additive explanation methods like additive SHAP \citep{lundberg2017unified}, which decompose predictions of complex models into additive contributions. However, these methods provide approximations of the underlying black-box model's behavior rather than inherently interpretable architectures, and they can be computationally expensive for large models.
In contrast, NAEs build additivity directly into the architecture and training objective. Because the model is defined as a sum of per-feature expert contributions, the resulting explanations are exact by construction rather than approximations of a more complex underlying function.
Furthermore, unlike post-hoc methods that can also be applied to trained models, NAEs allow explicit control over the models during training via the expert-variation penalty $\lambda$. This parameter governs the deviation from strict additivity, enabling a direct trade-off between expressivity and interpretability. Additionally, NAE feature attributions are inherent to the forward pass, avoiding the computational overhead of estimating SHAP values for complex models.

\subsection{Mixture of Experts}

The Mixture of Experts (MoE) paradigm introduces a modular approach to modeling, where multiple specialized sub-models (experts) are trained and a gating network dynamically selects or combines their outputs \citep{jacobs1991adaptive}. MoE architectures have demonstrated remarkable scalability and adaptability, particularly in large-scale neural networks. Notably, sparse gating mechanisms have enabled efficient training of massive models by activating only a subset of experts per input \citep{shazeer2017outrageously}. MoE has been successfully integrated into Transformer architectures \citep{vaswani2017attention}, powering advances in large language models \citep{artetxe2022efficient,jiang2024mixtral,lepikhin2020gshard} and multimodal systems \citep{lin2024moe,shen2023scaling}. 

While MoE models excel at capturing diverse patterns and scaling to large datasets, they typically lack the interpretability of additive models.
Specifically, standard MoE architectures employ experts that operate on the full input vector, with a gating network combining their outputs into a single prediction. While flexible, this approach sacrifices the decomposability required for interpretability. In contrast, NAEs organize experts per feature, ensuring that the final prediction remains a sum of feature-wise contributions (Equation \eqref{eq:final_prediction}) even when the gating network utilizes all features. Similarly, unlike general modular networks which focus on functional specialization without enforcing additive constraints, NAEs enforce modularization specifically at the feature level. This design, combined with our expert-variation penalty, allows NAEs to function as a structured, interpretable variant of MoE that retains GAM-like additivity while offering explicit control over interaction strength.

\subsection{Integrating Additive Models with Mixture of Experts}

Several efforts have sought to combine the interpretability of GAMs with the flexibility of MoE architectures, but important limitations remain. The Generalized Additive Models for Location, Scale, and Shape (GAMLSS) framework \citep{rigby2005generalized} extends GAMs to model multiple distributional parameters (e.g., mean and variance) in an additive manner, increasing flexibility but still failing to capture complex feature interactions. FlexMix \citep{leisch2004flexmix} introduces a mixture model of GAMs with expert weights estimated via Expectation-Maximization, yet these weights are fixed after training and cannot adapt to new inputs, limiting expressiveness. Mixdistreg \citep{rugamer2023mixdistreg,rugamer2024mixture} and MNAM \citep{kim2024generalizing} propose a dynamic mixture of additive models with input-dependent weights, but apply the same set of weights uniformly across all features or experts, preventing feature-specific adaptation and lacking explicit mechanisms to balance interpretability and accuracy.

In summary, while prior work has made progress toward integrating interpretability and accuracy, existing methods either retain strict additivity or lack fine-grained control over the trade-off between accuracy and transparency. This motivates the development of new frameworks, such as NAEs, that provide both feature-level interpretability and adaptive modeling capacity. By enabling context-dependent expert selection and feature-specific adaptation, our model bridges the gap between transparent, interpretable models and highly flexible, accurate architectures, opening new possibilities for trustworthy machine learning in complex domains.

\section{CONCLUSION}

We introduced Neural Additive Experts (NAEs), a novel framework that balances interpretability and predictive accuracy by extending additive models with a mixture-of-experts architecture and dynamic gating. NAEs enable each feature to be modeled by multiple specialized networks, with a gating mechanism that adaptively integrates information across features, thus relaxing strict additive constraints. Our theoretical analysis demonstrates that NAEs can recover complex, context-dependent relationships that are inaccessible to standard GAMs, while targeted regularization provides explicit control over the trade-off between flexibility and transparency. Empirical results on both synthetic and real-world datasets show that NAEs achieve competitive accuracy with state-of-the-art models, while maintaining clear, feature-level explanations. This work provides a principled and practical approach for applications where both interpretability and accuracy are essential.

\subsubsection*{Acknowledgements}
This work is supported in part by the US National Science Foundation (NSF) and the National Institute of Health (NIH) under grants IIS-2106913, IIS-2538206, IIS-2529378, CCF-2217071, CNS-2213700, and R01LM014012-01A1. Any recommendations expressed in this material are those of the authors and do not necessarily reflect the views of NIH or NSF.

\bibliographystyle{apalike}

\bibliography{reference}

@article{agarwal2021neural,
  title={Neural additive models: Interpretable machine learning with neural nets},
  author={Agarwal, Rishabh and Melnick, Levi and Frosst, Nicholas and Zhang, Xuezhou and Lengerich, Ben and Caruana, Rich and Hinton, Geoffrey E},
  journal={Advances in Neural Information Processing Systems},
  volume={34},
  pages={4699--4711},
  year={2021}
}

@inproceedings{chang2022node,
title={{NODE}-{GAM}: Neural Generalized Additive Model for Interpretable Deep Learning},
author={Chun-Hao Chang and Rich Caruana and Anna Goldenberg},
booktitle={International Conference on Learning Representations},
year={2022},
url={https://openreview.net/forum?id=g8NJR6fCCl8}
}

@inproceedings{radenovic2022neural,
 author = {Radenovic, Filip and Dubey, Abhimanyu and Mahajan, Dhruv},
 booktitle = {Advances in Neural Information Processing Systems},
 editor = {S. Koyejo and S. Mohamed and A. Agarwal and D. Belgrave and K. Cho and A. Oh},
 pages = {8414--8426},
 publisher = {Curran Associates, Inc.},
 title = {Neural Basis Models for Interpretability},
 volume = {35},
 year = {2022}
}

@inproceedings{popov2020neural,
title={Neural Oblivious Decision Ensembles for Deep Learning on Tabular Data},
author={Sergei Popov and Stanislav Morozov and Artem Babenko},
booktitle={International Conference on Learning Representations},
year={2020},
url={https://openreview.net/forum?id=r1eiu2VtwH}
}

@inproceedings{lou2012intelligible,
  title={Intelligible models for classification and regression},
  author={Lou, Yin and Caruana, Rich and Gehrke, Johannes},
  booktitle={Proceedings of the 18th ACM SIGKDD international conference on Knowledge discovery and data mining},
  pages={150--158},
  year={2012}
}

@article{nori2019interpretml,
  title={Interpretml: A unified framework for machine learning interpretability},
  author={Nori, Harsha and Jenkins, Samuel and Koch, Paul and Caruana, Rich},
  journal={arXiv preprint arXiv:1909.09223},
  year={2019}
}

@inproceedings{chang2021interpretable,
  title={How interpretable and trustworthy are GAMs?},
  author={Chang, Chun-Hao and Tan, Sarah and Lengerich, Ben and Goldenberg, Anna and Caruana, Rich},
  booktitle={Proceedings of the 27th ACM SIGKDD Conference on Knowledge Discovery \& Data Mining},
  pages={95--105},
  year={2021}
}

@article{vaswani2017attention,
  title={Attention is all you need},
  author={Vaswani, Ashish and Shazeer, Noam and Parmar, Niki and Uszkoreit, Jakob and Jones, Llion and Gomez, Aidan N and Kaiser, {\L}ukasz and Polosukhin, Illia},
  journal={Advances in neural information processing systems},
  volume={30},
  year={2017}
}

@incollection{hastie2017generalized,
  title={Generalized additive models},
  author={Hastie, Trevor J},
  booktitle={Statistical models in S},
  pages={249--307},
  year={2017},
  publisher={Routledge}
}

@inproceedings{chen2016xgboost,
  title={Xgboost: A scalable tree boosting system},
  author={Chen, Tianqi and Guestrin, Carlos},
  booktitle={Proceedings of the 22nd acm sigkdd international conference on knowledge discovery and data mining},
  pages={785--794},
  year={2016}
}

@article{saeed2011multiparameter,
  title={Multiparameter Intelligent Monitoring in Intensive Care II (MIMIC-II): a public-access intensive care unit database},
  author={Saeed, Mohammed and Villarroel, Mauricio and Reisner, Andrew T and Clifford, Gari and Lehman, Li-Wei and Moody, George and Heldt, Thomas and Kyaw, Tin H and Moody, Benjamin and Mark, Roger G},
  journal={Critical care medicine},
  volume={39},
  number={5},
  pages={952},
  year={2011},
  publisher={NIH Public Access}
}

@article{johnson2016mimic,
  title={MIMIC-III, a freely accessible critical care database},
  author={Johnson, Alistair EW and Pollard, Tom J and Shen, Lu and Lehman, Li-wei H and Feng, Mengling and Ghassemi, Mohammad and Moody, Benjamin and Szolovits, Peter and Anthony Celi, Leo and Mark, Roger G},
  journal={Scientific data},
  volume={3},
  number={1},
  pages={1--9},
  year={2016},
  publisher={Nature Publishing Group}
}

@article{blake1998uci,
  title={UCI repository of machine learning databases},
  author={Blake, Catherine},
  journal={http://www. ics. uci. edu/\~{} mlearn/MLRepository. html},
  year={1998},
  publisher={University of California, Department of Information and Computer Science}
}

@article{pace1997sparse,
  title={Sparse spatial autoregressions},
  author={Pace, R Kelley and Barry, Ronald},
  journal={Statistics \& Probability Letters},
  volume={33},
  number={3},
  pages={291--297},
  year={1997},
  publisher={Elsevier}
}

@article{jacobs1991adaptive,
  title={Adaptive mixtures of local experts},
  author={Jacobs, Robert A and Jordan, Michael I and Nowlan, Steven J and Hinton, Geoffrey E},
  journal={Neural computation},
  volume={3},
  number={1},
  pages={79--87},
  year={1991},
  publisher={MIT Press}
}

@inproceedings{shazeer2017outrageously,
title={ Outrageously Large Neural Networks: The Sparsely-Gated Mixture-of-Experts Layer},
author={Noam Shazeer and *Azalia Mirhoseini and *Krzysztof Maziarz and Andy Davis and Quoc Le and Geoffrey Hinton and Jeff Dean},
booktitle={International Conference on Learning Representations},
year={2017},
url={https://openreview.net/forum?id=B1ckMDqlg}
}

@inproceedings{lepikhin2020gshard,
title={{\{}GS{\}}hard: Scaling Giant Models with Conditional Computation and Automatic Sharding},
author={Dmitry Lepikhin and HyoukJoong Lee and Yuanzhong Xu and Dehao Chen and Orhan Firat and Yanping Huang and Maxim Krikun and Noam Shazeer and Zhifeng Chen},
booktitle={International Conference on Learning Representations},
year={2021},
url={https://openreview.net/forum?id=qrwe7XHTmYb}
}

@article{jiang2024mixtral,
  title={Mixtral of experts},
  author={Albert Q. Jiang and Alexandre Sablayrolles and Antoine Roux and Arthur Mensch and Blanche Savary and Chris Bamford and Devendra Singh Chaplot and Diego de las Casas and Emma Bou Hanna and Florian Bressand and Gianna Lengyel and Guillaume Bour and Guillaume Lample and Lélio Renard Lavaud and Lucile Saulnier and Marie-Anne Lachaux and Pierre Stock and Sandeep Subramanian and Sophia Yang and Szymon Antoniak and Teven Le Scao and Théophile Gervet and Thibaut Lavril and Thomas Wang and Timothée Lacroix and William El Sayed},
  journal={arXiv preprint arXiv:2401.04088},
  year={2024}
}

@inproceedings{artetxe2022efficient,
    title = "Efficient Large Scale Language Modeling with Mixtures of Experts",
    author = "Artetxe, Mikel  and
      Bhosale, Shruti  and
      Goyal, Naman  and
      Mihaylov, Todor  and
      Ott, Myle  and
      Shleifer, Sam  and
      Lin, Xi Victoria  and
      Du, Jingfei  and
      Iyer, Srinivasan  and
      Pasunuru, Ramakanth  and
      Anantharaman, Giridharan  and
      Li, Xian  and
      Chen, Shuohui  and
      Akin, Halil  and
      Baines, Mandeep  and
      Martin, Louis  and
      Zhou, Xing  and
      Koura, Punit Singh  and
      O{'}Horo, Brian  and
      Wang, Jeffrey  and
      Zettlemoyer, Luke  and
      Diab, Mona  and
      Kozareva, Zornitsa  and
      Stoyanov, Veselin",
    editor = "Goldberg, Yoav  and
      Kozareva, Zornitsa  and
      Zhang, Yue",
    booktitle = "Proceedings of the 2022 Conference on Empirical Methods in Natural Language Processing",
    month = dec,
    year = "2022",
    address = "Abu Dhabi, United Arab Emirates",
    publisher = "Association for Computational Linguistics",
    url = "https://aclanthology.org/2022.emnlp-main.804/",
    doi = "10.18653/v1/2022.emnlp-main.804",
    pages = "11699--11732",
}

@inproceedings{shen2023scaling,
  title={Scaling Vision-Language Models with Sparse Mixture of Experts},
  author={Shen, Sheng and Yao, Zhewei and Li, Chunyuan and Darrell, Trevor and Keutzer, Kurt and He, Yuxiong},
  booktitle={Findings of the Association for Computational Linguistics: EMNLP 2023},
  pages={11329--11344},
  year={2023}
}

@article{lin2024moe,
  title={Moe-llava: Mixture of experts for large vision-language models},
  author={Lin, Bin and Tang, Zhenyu and Ye, Yang and Cui, Jiaxi and Zhu, Bin and Jin, Peng and Zhang, Junwu and Ning, Munan and Yuan, Li},
  journal={arXiv preprint arXiv:2401.15947},
  year={2024}
}

@inproceedings{jang2016categorical,
  title={Categorical Reparameterization with Gumbel-Softmax},
  author={Jang, Eric and Gu, Shixiang and Poole, Ben},
  booktitle={International Conference on Learning Representations},
  year={2016}
}

@article{zhang2024gaussian, title={Gaussian Process Neural Additive Models}, volume={38}, url={https://ojs.aaai.org/index.php/AAAI/article/view/29628}, DOI={10.1609/aaai.v38i15.29628}, abstractNote={Deep neural networks have revolutionized many fields, but their black-box nature also occasionally prevents their wider adoption in fields such as healthcare and finance where interpretable and explainable models are required. The recent development of Neural Additive Models (NAMs) poses a major step in the direction of interpretable deep learning for tabular datasets. In this paper, we propose a new subclass of NAMs that utilize a single-layer neural network construction of the Gaussian process via random Fourier features, which we call Gaussian Process Neural Additive Models (GP-NAM). GP-NAMs have the advantage of a convex objective function and number of trainable parameters that grows linearly with feature dimensions. It suffers no loss in performance compared with deeper NAM approaches because GPs are well-suited to learning complex non-parametric univariate functions. We demonstrate the performance of GP-NAM on several tabular datasets, showing that it achieves comparable performance in both classification and regression tasks with a massive reduction in the number of parameters.}, number={15}, journal={Proceedings of the AAAI Conference on Artificial Intelligence}, author={Zhang, Wei and Barr, Brian and Paisley, John}, year={2024}, month={Mar.}, pages={16865-16872} }

@inproceedings{thielmann2024neural,
  title={Neural additive models for location scale and shape: A framework for interpretable neural regression beyond the mean},
  author={Thielmann, Anton Frederik and Kruse, Ren{\'e}-Marcel and Kneib, Thomas and S{\"a}fken, Benjamin},
  booktitle={International Conference on Artificial Intelligence and Statistics},
  pages={1783--1791},
  year={2024},
  organization={PMLR}
}

@article{bouchiat2023improving,
  title={Improving Neural Additive Models with Bayesian Principles},
  author={Bouchiat, Kouroche and Immer, Alexander and Y{\`e}che, Hugo and R{\"a}tsch, Gunnar and Fortuin, Vincent},
  year={2023}
}

@inproceedings{loshchilov2018decoupled,
  title={Decoupled Weight Decay Regularization},
  author={Loshchilov, Ilya and Hutter, Frank},
  booktitle={International Conference on Learning Representations},
  year={2018}
}

@article{rigby2005generalized,
  title={Generalized additive models for location, scale and shape},
  author={Rigby, Robert A and Stasinopoulos, D Mikis},
  journal={Journal of the Royal Statistical Society Series C: Applied Statistics},
  volume={54},
  number={3},
  pages={507--554},
  year={2005},
  publisher={Oxford University Press}
}

@article{leisch2004flexmix,
  title={Flexmix: A general framework for finite mixture models and latent glass regression in R},
  author={Leisch, Friedrich},
  year={2004}
}

@article{rugamer2023mixdistreg,
  title={mixdistreg: An R Package for Fitting Mixture of Experts Distributional Regression with Adaptive First-order Methods},
  author={R{\"u}gamer, David},
  journal={arXiv preprint arXiv:2302.02043},
  year={2023}
}

@article{rugamer2024mixture,
  title={Mixture of experts distributional regression: implementation using robust estimation with adaptive first-order methods},
  author={R{\"u}gamer, David and Pfisterer, Florian and Bischl, Bernd and Gr{\"u}n, Bettina},
  journal={AStA Advances in Statistical Analysis},
  volume={108},
  number={2},
  pages={351--373},
  year={2024},
  publisher={Springer}
}

@article{enouen2022sparse,
  title={Sparse interaction additive networks via feature interaction detection and sparse selection},
  author={Enouen, James and Liu, Yan},
  journal={Advances in Neural Information Processing Systems},
  volume={35},
  pages={13908--13920},
  year={2022}
}

@inproceedings{hegselmann2020evaluation,
  title={An evaluation of the doctor-interpretability of generalized additive models with interactions},
  author={Hegselmann, Stefan and Volkert, Thomas and Ohlenburg, Hendrik and Gottschalk, Antje and Dugas, Martin and Ertmer, Christian},
  booktitle={Machine Learning for Healthcare Conference},
  pages={46--79},
  year={2020},
  organization={PMLR}
}

@article{pedersen2019hierarchical,
  title={Hierarchical generalized additive models in ecology: an introduction with mgcv},
  author={Pedersen, Eric J and Miller, David L and Simpson, Gavin L and Ross, Noam},
  journal={PeerJ},
  volume={7},
  pages={e6876},
  year={2019},
  publisher={PeerJ Inc.}
}

@article{tan2018learning,
  title={Learning global additive explanations for neural nets using model distillation},
  author={Tan, Sarah and Caruana, Rich and Hooker, Giles and Koch, Paul and Gordo, Albert},
  journal={stat},
  volume={1050},
  pages={3},
  year={2018}
}

@inproceedings{duong2024cat,
  title={Cat: Interpretable concept-based taylor additive models},
  author={Duong, Viet and Wu, Qiong and Zhou, Zhengyi and Zhao, Hongjue and Luo, Chenxiang and Zavesky, Eric and Yao, Huaxiu and Shao, Huajie},
  booktitle={Proceedings of the 30th ACM SIGKDD Conference on Knowledge Discovery and Data Mining},
  pages={723--734},
  year={2024}
}

@article{miller2019explanation,
  title={Explanation in artificial intelligence: Insights from the social sciences},
  author={Miller, Tim},
  journal={Artificial intelligence},
  volume={267},
  pages={1--38},
  year={2019},
  publisher={Elsevier}
}

@inproceedings{lou2013accurate,
  title={Accurate intelligible models with pairwise interactions},
  author={Lou, Yin and Caruana, Rich and Gehrke, Johannes and Hooker, Giles},
  booktitle={Proceedings of the 19th ACM SIGKDD international conference on Knowledge discovery and data mining},
  pages={623--631},
  year={2013}
}

@article{kim2024generalizing,
  title={Generalizing Neural Additive Models via Statistical Multimodal Analysis},
  author={Kim, Young Kyung and Di Martino, Juan Matias and Sapiro, Guillermo},
  journal={Transactions on Machine Learning Research},
  year={2024}
}

@article{mueller2024gamformer,
  title={GAMformer: In-Context Learning for Generalized Additive Models},
  author={Mueller, Andreas and Siems, Julien and Nori, Harsha and Salinas, David and Zela, Arber and Caruana, Rich and Hutter, Frank},
  journal={arXiv preprint arXiv:2410.04560},
  year={2024}
}

@inproceedings{cui2020factored,
  title={A factored generalized additive model for clinical decision support in the operating room},
  author={Cui, Zhicheng and Fritz, Bradley A and King, Christopher R and Avidan, Michael S and Chen, Yixin},
  booktitle={AMIA Annual Symposium Proceedings},
  volume={2019},
  pages={343},
  year={2020}
}

@article{ibrahim2023grand,
  title={GRAND-SLAMIN’Interpretable Additive Modeling with Structural Constraints},
  author={Ibrahim, Shibal and Afriat, Gabriel and Behdin, Kayhan and Mazumder, Rahul},
  journal={Advances in Neural Information Processing Systems},
  volume={36},
  pages={61158--61186},
  year={2023}
}

@article{ibrahim2025predicting,
  title={Predicting census survey response rates with parsimonious additive models and structured interactions},
  author={Ibrahim, Shibal and Radchenko, Peter and Ben-David, Emanuel and Mazumder, Rahul},
  journal={The Annals of Applied Statistics},
  volume={19},
  number={1},
  pages={94--120},
  year={2025},
  publisher={Institute of Mathematical Statistics}
}

@article{mcintosh2025generalized,
  title={Generalized additive logit models for clinical prediction},
  author={McIntosh, Cameron Norman},
  journal={Blood Advances},
  volume={9},
  number={4},
  pages={935--936},
  year={2025},
  publisher={American Society of Hematology Washington, DC}
}

@inproceedings{biran2017explanation,
  title={Explanation and justification in machine learning: A survey},
  author={Biran, Or and Cotton, Courtenay},
  booktitle={IJCAI-17 workshop on explainable AI (XAI)},
  volume={8},
  number={1},
  pages={8--13},
  year={2017}
}

@misc{ruppert2004elements,
  title={The elements of statistical learning: data mining, inference, and prediction},
  author={Ruppert, David},
  year={2004},
  publisher={Taylor \& Francis}
}

@inproceedings{caruana2015intelligible,
  title={Intelligible models for healthcare: Predicting pneumonia risk and hospital 30-day readmission},
  author={Caruana, Rich and Lou, Yin and Gehrke, Johannes and Koch, Paul and Sturm, Marc and Elhadad, Noemie},
  booktitle={Proceedings of the 21th ACM SIGKDD international conference on knowledge discovery and data mining},
  pages={1721--1730},
  year={2015}
}

@article{christoph2020interpretable,
  title={Interpretable machine learning: A guide for making black box models explainable},
  author={Christoph, Molnar},
  year={2020},
  publisher={Leanpub}
}

@inproceedings{lundberg2017unified,
 author = {Lundberg, Scott M and Lee, Su-In},
 booktitle = {Advances in Neural Information Processing Systems},
 editor = {I. Guyon and U. Von Luxburg and S. Bengio and H. Wallach and R. Fergus and S. Vishwanathan and R. Garnett},
 pages = {},
 publisher = {Curran Associates, Inc.},
 title = {A Unified Approach to Interpreting Model Predictions},
 url = {https://proceedings.neurips.cc/paper_files/paper/2017/file/8a20a8621978632d76c43dfd28b67767-Paper.pdf},
 volume = {30},
 year = {2017}
}

@article{goldstein2015peeking,
  title={Peeking inside the black box: Visualizing statistical learning with plots of individual conditional expectation},
  author={Goldstein, Alex and Kapelner, Adam and Bleich, Justin and Pitkin, Emil},
  journal={journal of Computational and Graphical Statistics},
  volume={24},
  number={1},
  pages={44--65},
  year={2015},
  publisher={Taylor \& Francis}
}

@inproceedings{gkolemis2023rhale,
  author={Vasilis Gkolemis and Theodore Dalamagas and Eirini Ntoutsi and Christos Diou},
  title={RHALE: Robust and Heterogeneity-Aware Accumulated Local Effects},
  year={2023},
  cdate={1672531200000},
  pages={859-866},
  url={https://doi.org/10.3233/FAIA230354},
  booktitle={ECAI},
}

@article{xiong2025protonam,
author = {Xiong, Guangzhi and Sinha, Sanchit and Zhang, Aidong},
title = {ProtoNAM: Prototypical Neural Additive Models for Interpretable Deep Tabular Learning},
year = {2025},
issue_date = {November 2025},
publisher = {Association for Computing Machinery},
address = {New York, NY, USA},
volume = {19},
number = {9},
issn = {1556-4681},
url = {https://doi.org/10.1145/3766072},
doi = {10.1145/3766072},
journal = {ACM Trans. Knowl. Discov. Data},
month = oct,
articleno = {163},
numpages = {25}
}

%%%%%%%%%%%%%%%%%%%%%%%%%%%%%%%%%%%%%%%%%%%%%%%%%%%%%%%%%%%%

\section*{Checklist}

\begin{enumerate}

  \item For all models and algorithms presented, check if you include:
  \begin{enumerate}
    \item A clear description of the mathematical setting, assumptions, algorithm, and/or model. [Yes]
    \item An analysis of the properties and complexity (time, space, sample size) of any algorithm. [Yes]
    \item (Optional) Anonymized source code, with specification of all dependencies, including external libraries. [Yes]
  \end{enumerate}

  \item For any theoretical claim, check if you include:
  \begin{enumerate}
    \item Statements of the full set of assumptions of all theoretical results. [Yes]
    \item Complete proofs of all theoretical results. [Yes]
    \item Clear explanations of any assumptions. [Yes]     
  \end{enumerate}

  \item For all figures and tables that present empirical results, check if you include:
  \begin{enumerate}
    \item The code, data, and instructions needed to reproduce the main experimental results (either in the supplemental material or as a URL). [Yes]
    \item All the training details (e.g., data splits, hyperparameters, how they were chosen). [Yes]
    \item A clear definition of the specific measure or statistics and error bars (e.g., with respect to the random seed after running experiments multiple times). [Yes]
    \item A description of the computing infrastructure used. (e.g., type of GPUs, internal cluster, or cloud provider). [Yes]
  \end{enumerate}

  \item If you are using existing assets (e.g., code, data, models) or curating/releasing new assets, check if you include:
  \begin{enumerate}
    \item Citations of the creator If your work uses existing assets. [Yes]
    \item The license information of the assets, if applicable. [Yes]
    \item New assets either in the supplemental material or as a URL, if applicable. [Yes]
    \item Information about consent from data providers/curators. [Not Applicable]
    \item Discussion of sensible content if applicable, e.g., personally identifiable information or offensive content. [Not Applicable]
  \end{enumerate}

  \item If you used crowdsourcing or conducted research with human subjects, check if you include:
  \begin{enumerate}
    \item The full text of instructions given to participants and screenshots. [Not Applicable]
    \item Descriptions of potential participant risks, with links to Institutional Review Board (IRB) approvals if applicable. [Not Applicable]
    \item The estimated hourly wage paid to participants and the total amount spent on participant compensation. [Not Applicable]
  \end{enumerate}

\end{enumerate}

\clearpage
\appendix
\thispagestyle{empty}

% Supplementary material: To improve readability, you must use a single-column format for the supplementary material.
\onecolumn
\aistatstitle{Context-Gated Experts for Controllable Model Additivity: \\
Supplementary Materials}

\section{DATASET DESCRIPTION} \label{sec:data_desc}

\paragraph{Housing Dataset} The California Housing dataset, cited from \citet{pace1997sparse}, encompasses a regression task based on median housing prices across California's census blocks. It comprises 20,640 instances, each characterized by eight distinct attributes.

\paragraph{MIMIC-II Dataset} The MIMIC-II (Multiparameter Intelligent Monitoring in Intensive Care) dataset, as described by \citet{saeed2011multiparameter}, facilitates a binary classification task aimed at predicting mortality in intensive care units (ICUs). It contains 17 attributes, including seven categorical variables.

\paragraph{MIMIC-III Dataset} MIMIC-III is a more extensive and detailed iteration of the MIMIC database referenced from \citet{johnson2016mimic}. In our study, the same setting in NODE-GAM \citep{chang2022node} is adopted, wherein categorical variables have been transformed into dummy variables for enhanced analysis.

\paragraph{Income Dataset} Originating from the UCI Machine Learning Repository \citep{blake1998uci}, the Income dataset underpins a binary classification task. Its objective is to predict whether an individual earns more than \$50,000 annually.

\paragraph{Credit Dataset} The Credit dataset\footnote{The Credit dataset can be downloaded at \url{https://www.kaggle.com/datasets/mlg-ulb/creditcardfraud}.} provides samples for a binary classification task on transaction fraud detection. It contains 30 anonymized features including 28 coefficients of PCA components. In Credit, 492 out of 284,807 transactions are labeled as frauds.

\paragraph{Year Dataset} The Year dataset\footnote{The Year dataset can be downloaded at \url{https://archive.ics.uci.edu/dataset/203/yearpredictionmsd}} contains data for a regression task, which uses the audio features to predict the release year of a song. It includes 515,345 samples with 90 features.

The statistics of each dataset can be found in Table \ref{tab:data_stat}.

\begin{table}[h!] \small
    \centering
    \caption{Statistics of real-world datasets used in model evaluation.}
    \begin{center}
    \begin{tabular}{lccccc}
        \toprule
        \bf Dataset & \bf Task & \bf \#Instances & \bf \#Features & \bf \#Categorical Features & \bf \#Classes\\
        \midrule
        Housing & regression & 20,640 & 8 & 0 & -- \\
        MIMIC-II & classification & 24,508 & 17 & 7 & 2 \\
        MIMIC-III & classification & 27,348 & 57 & 0 & 2 \\
        Income & classification & 32,561 & 14 & 8 & 2 \\
        Credit & classification & 284,807 & 30 & 0 & 2 \\
        Year & regression & 515,345 & 90 & 0 & -- \\
        \bottomrule
    \end{tabular}
\end{center}
    \label{tab:data_stat}
\end{table}

\section{IMPLEMENTATION DETAILS}
\label{sec:model_implement}

For a comprehensive and equitable evaluation, we employed the officially released codes for baseline models including NODE, NODE-GAM/NODE-GA\(^2\)M, EBM/EB\(^2\)M, and NBM/NB\(^2\)M. We adopted the PyTorch implementation of NAM/NA\(^2\)M by NBM, which has been benchmarked against their model, for a direct comparison with our proposed models. In developing NAEs, we constructed a multi-channel, fully connected network capable of encoding multiple features independently in a parallel manner.

For a fair comparison with existing methods, our experiments employ the processed version of MIMIC-II, MIMIC-III, Income, and Credit from \citet{chang2022node}, where the datasets are split into five parts for a five-fold cross-validation. The Housing dataset and the Year dataset are split into the training, validation, and test sets following setups of previous research \citep{chang2022node,radenovic2022neural}, and each of them is tested with 10 different seeds to examine the variation of predictions. Consistent with previous studies \citep{chang2022node,popov2020neural,radenovic2022neural}, we applied the same quantile transformation to the features. 

To be consistent with existing neural network-based GAMs \citep{agarwal2021neural,radenovic2022neural}, the feature encoders \(f_1,\cdots,f_n\) in NAE are implemented with MLPs. The expert predictors \(g_{11},\cdots,g_{nK}\) are implemented as one linear layer in our experiments.
During training, we decrease the learning rate with cosine annealing following NBM \citep{radenovic2022neural}. AdamW optimizer \citep{loshchilov2018decoupled} is used to optimize the training objective.
All experiments were run on NVIDIA A100 GPUs (40 GB / 80 GB).

\section{HYPERPARAMETERS} \label{sec:app_hyper}
\begin{table}[h!] \small
    \centering
    \caption{Hyperparameters for NAE on all datasets.}
    \begin{center}
    \begin{tabular}{c|ccccccc}
        \toprule
        Hyperparameter & Housing & MIMIC-II & MIMIC-III & Income & Credit & Year \\
        \midrule
        \#layers & 4 & 4 & 4 & 4 & 4 & 4 \\
        Hidden dimension & 128 & 128 & 128 & 128 & 128 & 128 \\
        \#total experts & 4 & 4 & 4 & 4 & 4 & 4 \\
        \#activated experts & 4 & 4 & 4 & 4 & 4 & 4 \\
        Batch size & 2048 & 2048 & 1024 & 2048 & 2048 & 512 \\
        Max iteration & 1000 & 1000 & 500 & 1000 & 150 & 75 \\
        Learning rate & 5.97e-4 & 1.97e-4 & 1.68e-4 & 1.11e-4 & 4.00e-5 & 1.17e-4 \\
        Weight decay & 5.29e-5 & 8.06e-4 & 2.78e-4 & 4.25e-3 & 1.88e-3 & 2.28e-6 \\
        Dropout & 0.1 & 0.0 & 0.1 & 0.0 & 0.3 & 0.1 \\
        Dropout expert & 0.2 & 0.4 & 0.2 & 0.5 & 0.2 & 0.6 \\
        Output Penalty & 1.97e-5 & 3.49e-8 & 4.86e-5 & 3.99e-1 & 5.07e-4 & 1.45e-4 \\
        Variation Penalty & 0.1 & 0.1 & 0.1 & 0.1 & 0.1 & 0.1 \\
        Normalization & layer\_norm & layer\_norm & layer\_norm & layer\_norm & batch\_norm & layer\_norm \\
        \bottomrule
    \end{tabular}
\end{center}
    \label{tab:hyerparams}
\end{table}

We perform random search over hyperparameters in the following ranges:

\begin{itemize}
    \item \#layers: the number of layers in each feature encoder, sampled from \{3, 4\}.
    \item Hidden dimension: the number of nodes in each layer of the feature encoder, sampled from \{64, 128\}.
    \item \#total experts: the number of total experts for each feature, sampled from \{4\}.
    \item \#activated experts: the number of activated experts for each feature, sampled from \{4\}.
    \item Batch size: the number of samples in one batch during training, sampled from \{512, 1024, 2048\}.
    \item Max iteration: the number of iterations for training, sampled from \{75, 150, 500, 1000\}.
    \item Learning rate: the speed of gradient descent, sampled from [1e-6, 1e-1].
    \item Weight decay: the coefficient for the L2 normalization on parameters, sampled from [1e-8, 1e-1].
    \item Dropout: the probability of a parameter in feature encoders being replaced as 0 during training, sampled from \{0.0, 0.1, 0.2, 0.3, 0.4, 0.5, 0.6, 0.7, 0.8, 0.9\}.
    \item Dropout expert: the probability of an expert's output being replaced as 0 during training, sampled from \{0.1, 0.2, 0.3, 0.4, 0.5, 0.6, 0.7, 0.8, 0.9\}.
    \item Output penalty: the coefficient for the L2 normalization on the outputs of features for reducing unimportant ones, sampled from [1e-8, 1e-1].
    \item Variation penalty: the weight for the expert variation penalty, sampled from \{0.1\}.
    \item Normalization: normalization methods used in the network, sampled from \{batch\_norm, layer\_norm\}.
\end{itemize}
The best hyperparameters we found for NAE are shown in Tables \ref{tab:hyerparams}.

\section{PROOFS FOR SECTION \ref{sec:theory}}
\label{app:proofs}

Throughout, each feature domain $\mathcal X_i$ is compact and metrizable, hence $\mathcal X_i$, $\mathcal X_i\times\mathcal X_j$ are compact Hausdorff. All functions are real-valued and continuous unless noted. We use $C(Y)$ for the Banach space of continuous real-valued functions on a compact space $Y$ with the sup norm $\|\cdot\|_\infty$.

\begin{lemma}[Restatement of Lemma \ref{lem:sep}]
\label{app:lem:sep}
Let $f_{ij}\in C(\mathcal X_i\times \mathcal X_j)$. For any $\varepsilon>0$ there exist $M\in\mathbb N$ and functions $\{u_m\in C(\mathcal X_i),\, v_m\in C(\mathcal X_j)\}_{m=1}^M$ such that
\[
\big\| f_{ij} - \sum_{m=1}^M u_m(\cdot)\,v_m(\cdot) \big\|_\infty < \varepsilon .
\]
\end{lemma}

\begin{proof}
Consider the set
\[
\mathcal A \;=\; \Big\{\, \sum_{m=1}^M u_m(x_i)\,v_m(x_j) \;:\; M\!\in\!\mathbb N,\ u_m\!\in\! C(\mathcal X_i),\ v_m\!\in\! C(\mathcal X_j)\,\Big\}.
\]
$\mathcal A$ is a subalgebra of $C(\mathcal X_i\times \mathcal X_j)$: it is closed under pointwise addition and multiplication and contains constants (take $u\equiv 1$, $v\equiv 1$). It separates points: given $(x_i,x_j)\neq(x'_i,x'_j)$, either $x_i\neq x'_i$ or $x_j\neq x'_j$. If $x_i\neq x'_i$, choose $u\in C(\mathcal X_i)$ with $u(x_i)\neq u(x'_i)$ and $v\equiv 1$. If $x_j\neq x'_j$, choose $v$ analogously and $u\equiv 1$. By Stone–Weierstrass, the uniform closure of $\mathcal A$ is $C(\mathcal X_i\times \mathcal X_j)$. Thus every $f_{ij}$ can be approximated arbitrarily well by finite sums of products $u_m v_m$.
\end{proof}

\begin{lemma}[Restatement of Lemma \ref{lem:two-expert}]
\label{app:lem:two-expert}
Let $u\in C(\mathcal X_i)$ and $v\in C(\mathcal X_j)$ be bounded. Choose any constant $C>\|v\|_\infty$. Define two experts (for feature $i$)
\[
g_{i,+}\!\circ E_i(x_i) = +C\,u(x_i),\quad
g_{i,-}\!\circ E_i(x_i) = -C\,u(x_i),
\]
and a two-class softmax gate with logits
\[
\phi_i^{(+)}(x) = \alpha(x)-\beta(x_j),\quad
\phi_i^{(-)}(x) = \alpha(x)+\beta(x_j).
\]
\end{lemma}

\begin{proof}
With two-class softmax, if the logits are $\alpha\pm\beta$, the probabilities are
\[
r_{i,+}(x)=\frac{1}{1+e^{2\beta(x_j)}},\quad
r_{i,-}(x)=\frac{1}{1+e^{-2\beta(x_j)}},
\]
so $r_{i,+}(x)-r_{i,-}(x)=-\tanh(\beta(x_j))$. Setting $\beta(x_j)=-\operatorname{arctanh}(v(x_j)/C)$ yields $r_{i,+}(x)-r_{i,-}(x)=v(x_j)/C$. Then
\[
o_i(x)=r_{i,+}(x)\,C\,u(x_i)+r_{i,-}(x)\,(-C\,u(x_i))=u(x_i)\,v(x_j).
\]
The dependence on $x_j$ only is enforced by letting $\alpha$ be constant and $\beta$ a function of $x_j$ alone.
\end{proof}

\begin{theorem}[Restatement of Theorem \ref{thm:GAM}]
\label{app:thm:GAM}
For any $f\in \GAM$, there exists an NAE with $K=1$ per feature such that $\hat y(x)=f(x)$ for all $x$.
\end{theorem}

\begin{proof}
Set $K=1$, $r_{i1}(x)\equiv 1$, choose $g_{i1}\!\circ E_i = f_i$ and $\omega_0$ to match the intercept. Then $\hat y(x)=f(x)$.
\end{proof}

\begin{theorem}[Restatement of Theorem \ref{thm:ga2m}]
\label{app:thm:ga2m}
Let $f\in \GAAM$. For any $\epsilon>0$, there exists an NAE with context-only gating and finite experts $\{K_i\}$ such that $\|\hat y-f\|_\infty<\epsilon$. Moreover, one may take
\[
K_i \;\le\; 1+2\sum_{j\ne i} M_{ij},
\]
where $M_{ij}$ is the number of separable terms used in Lemma~\ref{lem:sep}.
\end{theorem}

\begin{proof}
Write $f(x)=\omega_0+\sum_i f_i(x_i)+\sum_{i<j} f_{ij}(x_i,x_j)$. Approximate each $f_{ij}$ by $\sum_m u_{ijm}(x_i)v_{ijm}(x_j)$ with error $<\epsilon/(2\binom{n}{2})$. Realize each $f_i$ via a single expert with constant gate. Realize each separable product $u_{ijm}(x_i)v_{ijm}(x_j)$ using Lemma \ref{lem:two-expert}. Summing all contributions gives $\hat y$ within $\epsilon$. Counting experts yields the bound on $K_i$.
\end{proof}

\begin{theorem}[Restatement of Theorem \ref{thm:additivity}]
\label{app:thm:additivity}
Let $\theta_\lambda$ minimize the penalized empirical objective in Equation \eqref{eq:loss}:
\[
\mathcal L_N(\theta;\lambda)
=\frac1N\sum_{t=1}^N L\!\big(y_t,\hat y_\theta(x_t)\big)
\;+\;
\frac{\lambda}{nNK}
\sum_{t=1}^N \sum_{i=1}^n \sum_{k=1}^K
\Big(o_{ik}(x_t)-\tfrac{1}{K}\sum_{\ell=1}^K o_{i\ell}(x_t)\Big)^2.
\]
Then the within-feature dispersion
$\mathsf{Pen}(\theta):=\frac{1}{nNK}\sum_{t,i,k}\big(o_{ik}(x_t)-\frac{1}{K}\sum_\ell o_{i\ell}(x_t)\big)^2$
is nonincreasing in $\lambda$, and $\mathsf{Pen}(\theta_\lambda)\to 0$ as $\lambda\to\infty$.
Consequently, any limit point of $\hat y_{\theta_\lambda}$ is additive; moreover the additivity metric $A(\theta_\lambda)$
from Equation \eqref{eq:additivity} satisfies $A(\theta_\lambda)\to 1$.
\end{theorem}

\begin{proof}
\textbf{(1) Monotonicity in $\lambda$.}
Fix $0\le \lambda_1<\lambda_2$ and let $\theta_{\lambda_1},\theta_{\lambda_2}$ be minimizers for $\lambda_1,\lambda_2$.
By optimality,
\[
\underbrace{\tfrac1N\sum_t L(y_t,\hat y_{\theta_{\lambda_2}}(x_t))+\lambda_2\,\mathsf{Pen}(\theta_{\lambda_2})}_{\mathcal L_N(\theta_{\lambda_2};\lambda_2)}
\;\le\;
\underbrace{\tfrac1N\sum_t L(y_t,\hat y_{\theta_{\lambda_1}}(x_t))+\lambda_2\,\mathsf{Pen}(\theta_{\lambda_1})}_{\mathcal L_N(\theta_{\lambda_1};\lambda_2)},
\]
and
\[
\mathcal L_N(\theta_{\lambda_1};\lambda_1)
\;\le\;
\mathcal L_N(\theta_{\lambda_2};\lambda_1)
=
\tfrac1N\sum_t L(y_t,\hat y_{\theta_{\lambda_2}}(x_t))+\lambda_1\,\mathsf{Pen}(\theta_{\lambda_2}).
\]
Subtracting the second inequality from the first yields
$(\lambda_2-\lambda_1)\mathsf{Pen}(\theta_{\lambda_2})\le (\lambda_2-\lambda_1)\mathsf{Pen}(\theta_{\lambda_1})$,
so $\mathsf{Pen}(\theta_{\lambda})$ is nonincreasing.

\smallskip
\textbf{(2) $\mathsf{Pen}(\theta_\lambda)\to 0$.}
Take a zero-penalty competitor $\tilde\theta$ that implements a GAM with $K=1$ (or $K>1$ with $g_{ik}\equiv g_i$ so all $o_{ik}$ are identical); this is feasible by Theorem \ref{thm:GAM} and Equations \eqref{eq:expert_pred} - \eqref{eq:final_prediction}. Then for every $\lambda$,
\[
\mathcal L_N(\theta_\lambda;\lambda)\ \le\ \mathcal L_N(\tilde\theta;\lambda)\ =\ \tfrac1N\sum_t L\big(y_t,\hat y_{\tilde\theta}(x_t)\big)\ =:C.
\]
Because the loss is nonnegative for the tasks considered (e.g., MSE, cross-entropy), we have
$\lambda\,\mathsf{Pen}(\theta_\lambda)\le C$, hence $\mathsf{Pen}(\theta_\lambda)\le C/\lambda \to 0$.

\smallskip
\textbf{(3) Collapse within each feature head at training points.}
Fix any feature $i$ and training example $x_t$.
By definition of $\mathsf{Pen}$ and Step (2),
\[
\frac{1}{K}\sum_{k=1}^K\big(o_{ik}(x_t)-\bar o_i(x_t)\big)^2 \ \longrightarrow\ 0
\quad\text{where}\quad
\bar o_i(x_t)=\tfrac{1}{K}\sum_{\ell=1}^K o_{i\ell}(x_t).
\]
Thus the vector $\big(o_{i1}(x_t),\dots,o_{iK}(x_t)\big)$ becomes constant in $k$:
for each $(i,t)$ there exists a scalar $q_i(x_t)$ with $o_{ik}(x_t)\to q_i(x_t)$ for all $k$.
Since the feature contribution is $o_i(x_t)=\sum_k r_{ik}(x_t)\,o_{ik}(x_t)$ (Equation \eqref{eq:expert_aggregation}),
we get $o_i(x_t)\to \sum_k r_{ik}(x_t)\,q_i(x_t)=K\,q_i(x_t)$.
Define $h_i(x_{i,t}):=K\,q_i(x_t)$; then
\[
\hat y_{\theta_\lambda}(x_t)
=\omega_0+\sum_{i=1}^n o_i(x_t)
\ \Longrightarrow\ 
\omega_0^\star+\sum_{i=1}^n h_i(x_{i,t}),
\]
which is additive on the training set.

\smallskip
\textbf{(4) Additivity in the limit and the metric $A(\lambda)$.}
At any \emph{exact} zero-penalty parameter setting, $o_{ik}(x_t)$ is identical across $k$ for every $(i,t)$,
so setting $g_{ik}\equiv g_i$ and any (e.g., constant) softmax weights $r_{ik}$ yields a GAM realization
with the same predictions (Equations \eqref{eq:expert_pred} - \eqref{eq:final_prediction}).
For the additivity metric in Equation \eqref{eq:additivity}, when each $o_i$ depends only on $x_i$, we have
$\mathrm{Var}(\mathbb E[o_i\!\mid\!x_i])=\mathrm{Var}(o_i)$, hence $A(\theta_\lambda)\to 1$ as $\lambda\to\infty$.
\end{proof}

\section{MORE SIMULATION STUDIES} \label{sec:simu_more}

\subsection{Impact of Feature Sparsity}

To assess robustness to imbalanced feature distributions, we simulate datasets where the proportion of instances with \(x_2=1\) is set to 50\%, 25\%, 5\%, and 1\%. Figure~\ref{fig:simul_sparse} shows shape plots for traditional GAMs and NAEs with increasing numbers of experts (\(K=2, 8, 32\)).

\begin{figure}[h!]
    \centering
    \includegraphics[width=0.8\linewidth]{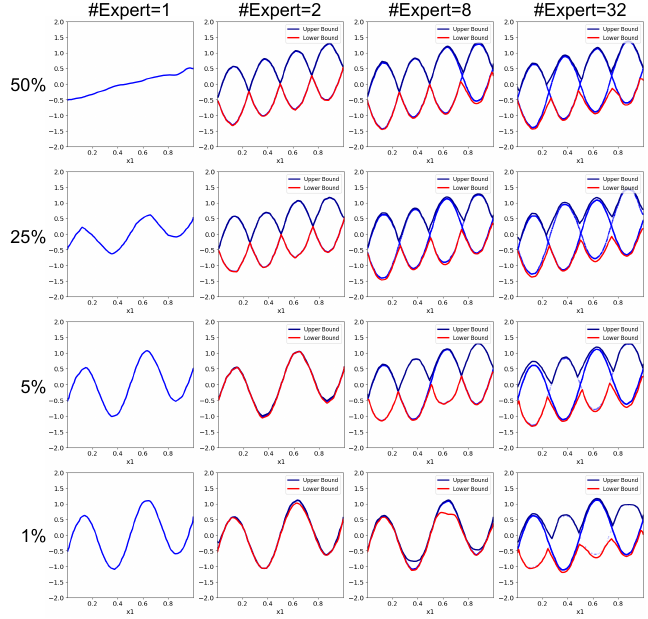}
    \caption{Shape plots of GAM and NAEs on simulated data with varying feature sparsity for \(x_2\). NAEs recover minority signals even at high sparsity, while GAM overfits the majority class.}
    \label{fig:simul_sparse}
\end{figure}

As sparsity increases, traditional GAMs overfit the majority class and miss minority patterns. NAEs, with more experts, consistently recover the multimodal structure, demonstrating robustness to feature imbalance.

\subsection{Impact of Distribution Modality}

We further test NAEs on data with increasing numbers of modes by introducing multiple categorical features:
\begin{equation}
y = \varepsilon + x_1 - \frac{1}{2} + \frac{1}{CF}\sum_{i=2}^{CF+1}x_i\sin(4\pi x_1),
\end{equation}
where \(\varepsilon \sim \mathcal{N}(0,0.01^2)\), \(x_1 \sim U(0,1)\), and each \(x_i\) (\(i\geq2\)) is uniformly sampled from \(\{-1,+1\}\). We vary the number of categorical features (\(CF = 1, 3, 5, 7\)).

\begin{figure}[h!]
    \centering
    \includegraphics[width=0.9\linewidth]{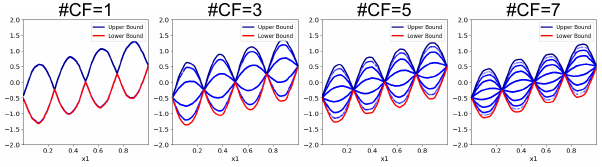}
    \caption{Shape plots of NAEs on simulated data with increasing numbers of categorical features. The model adapts to capture the growing complexity of the underlying function.}
    \label{fig:simul_variable}
\end{figure}

Figure~\ref{fig:simul_variable} shows that NAEs flexibly capture the increasing complexity of the shape function as the number of modes grows, highlighting the model's scalability and adaptability to complex, multimodal distributions.

\subsection{Impact of Correlated Features}

To examine the effect of feature correlation, we simulate data where \(x_2\) is correlated with \(x_1\). Based on the multimodal setup in Equation \eqref{eq:multimodal}, we introduce correlation between $x_1$ and $x_2$ by setting $x_2$ as a binary variable $\in \{-1, 1\}$ with conditional probability
\begin{equation}
P(x_2=1|x_1) = \rho x_1 + (1-\rho)/2,
\end{equation}
where $\rho \in \{0, 0.5, 0.9\}$ controls the correlation strength. The first feature $x_1$ is drawn uniformly from [0, 1], and the response is generated as
\begin{equation}
    y = x_2\sin(4\pi x_1) + x_2 + \varepsilon.
\end{equation}

The comparison results (RMSE, lower is better) between NAE and NAM are shown in Table \ref{tab:rho_results}. As shown in the table, NAE maintains low RMSE across all correlation levels, whereas NAM performs significantly worse and only slightly improves with increased correlation. This demonstrates the robustness of NAE to correlated inputs.

\begin{table}[h]
\centering
\caption{Performance comparison under different correlation levels.}
\begin{center}
\begin{tabular}{lccc}
\toprule
Model & $\rho = 0.0$ & $\rho = 0.5$ & $\rho = 0.9$ \\
\midrule
NAE & 0.1066 & 0.1068 & 0.1090 \\
NAM & 0.7112 & 0.6829 & 0.6088 \\
\bottomrule
\end{tabular}
\end{center}
\label{tab:rho_results}
\end{table}

\subsection{Impact of Generic Interactions}

To evaluate the ability of NAE to capture generic interactions, we simulate data with a complex interaction structure, where $x_1,x_2 \sim Uniform[-1, 1]$ and the response is generated as
\begin{equation}
    y = 2 \sin(\pi x_1)\cos(\pi x_2) + 0.5 x_1^2 + 0.5 x_2^2 + \varepsilon.
\end{equation}
This target includes a multiplicative interaction term that cannot be decomposed additively. 
With the same evaluation settings as in previous simulations, we test both NAE and NAM on this data. The results show that NAE achieves a low RMSE of 0.1306, closely approximating the true function, while NAM fails to capture the interaction and yields a much higher RMSE of 1.0114. This illustrates that NAE can effectively learn complex interactions that are not strictly additive, as predicted by Theorem \ref{thm:ga2m}.

\section{COMPARISON BETWEEN NAE AND GA$^2$M}\label{sec:ga2m}

As demonstrated in Section \ref{sec:empirical_real}, NAE provides both individual feature attributions and visualizations of feature interactions. To compare with GA$^2$M, which emphasizes feature interactions, we conducted additional experiments using EB$^2$M. Specifically, we obtained single-feature attributions by summing all interaction terms associated with each feature. For a fair comparison, we visualized the actual effect of each feature value, including the interaction effects, and plotted the lower and upper bounds (i.e., the minimum and maximum effect of a feature value across all captured interactions).

\begin{figure}[h!]
    \centering
    \includegraphics[width=1\linewidth]{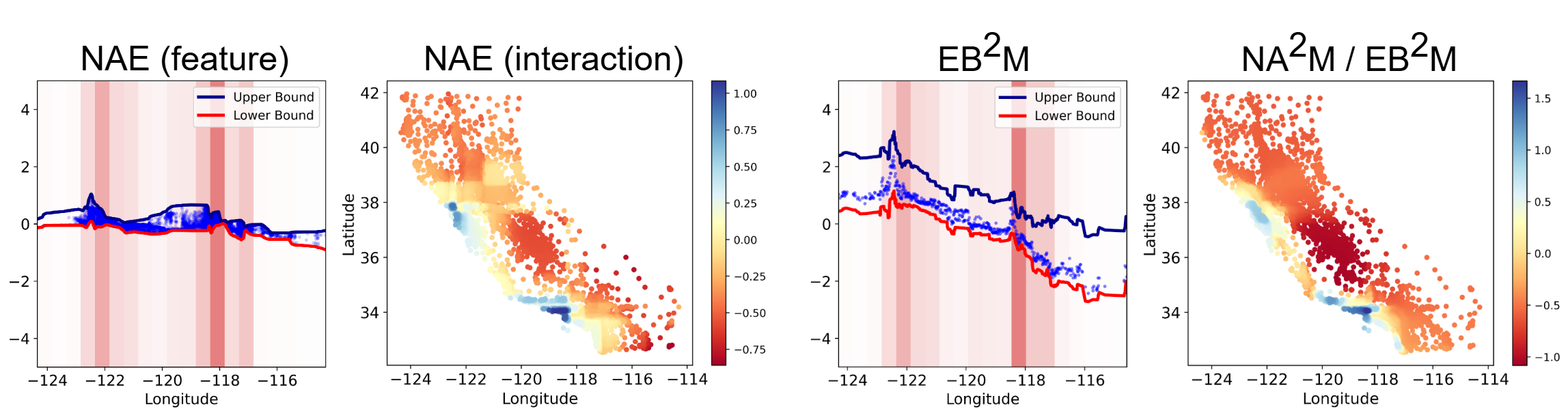}
    \caption{Comparison of NAE and GA$^2$M visualizations for the ``Longitude'' feature effect and the ``Longitude-Latitude'' interaction on house prices.}
    \label{fig:ga2m_comparison}
\end{figure}

As shown in Figure \ref{fig:ga2m_comparison}, while the interaction interpretation by GA$^2$M is similar to NAE, the single-feature attributions from EB$^2$M yield bounds that are too loose to be interpretable, suggesting high variability in predictions, especially for out-of-distribution inputs. Also, this approach to visualizing individual feature attributions is not applicable to NA$^2$M, as it does not provide strict output ranges. Moreover, a key advantage of NAE is its ability to control model additivity, allowing a smooth transition from a complex interaction model to a strictly additive one, which is a flexibility not available in GA$^2$M.

Additionally, NAE is not limited to capturing pairwise interactions. For example, one can modify the relevance estimation in Equation \eqref{eq:aggregation} to keep the interaction terms among three features. Figure \ref{fig:three_way} focuses on the Los Angeles area, demonstrating the effect of ``AveRooms'' on the price prediction beyond the Latitude-Longitude interaction. Interestingly, the south-western part of the area is more robust to the AveRooms change compared to other parts, revealing a complex three-way interaction that GA$^2$M fails to capture.

\begin{figure}[h!]
    \centering
    \includegraphics[width=0.9\linewidth]{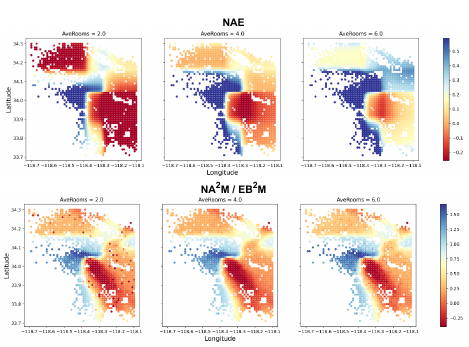}
    \caption{Comparison of NAE and GA$^2$M visualizations for the ``AveRooms-Latitude-Longitude'' interaction on house prices.}
    \label{fig:three_way}
\end{figure}

\section{NAE-D: NAE WITH DIAGONAL ROUTING MATRIX} \label{sec:mixnam-d}

We explored a special version of NAE, termed NAE-D, where the scoring matrices compose a block-diagonal matrix. In NAE-D, the matrix \(\mathcal{A}_{ij}\) in Equation \eqref{eq:aggregation} becomes a zero matrix if \(i\neq j\). This configuration allows NAE-D to function similarly to GAMs, where the relevance estimation of a feature's experts is solely based on its own encoded information.

Since the relevance for experts of each feature is solely determined by itself in NAE-D, the estimated relevance scores will be fixed for the same feature value. To encourage the model to learn various distributions with multiple experts, we introduce randomness into the decision-making process by employing the Gumbel-softmax technique \citep{jang2016categorical} with the temperature \(\tau=0.1\) to re-sample the experts during training based on their estimated relevance scores in Equation \eqref{eq:relevance}:

\begin{equation}
    \hat{r}_{ik} = \exp\Big(\frac{\log(r_{ik})+g_i}{\tau}\Big) \bigg/ \sum_{l=1}^K \exp\Big(\frac{\log(r_{il})+g_l}{\tau}\Big),
\end{equation}
where
\begin{equation}
g_i, g_l\sim Gumbel(0,1).
\end{equation}

The performance comparison of NAE-D with other baselines is presented in Table \ref{tab:mixnam-d_acc}. The results demonstrate that, by regularizing the scoring matrix to be diagonal, NAE-D shows a performance close to traditional additive models, which have worse accuracies than complex models that capture feature interactions. 

\begin{table*}[h!]  \small
    \centering
    \caption{Comparison of NAE-D to other baselines on benchmark datasets. ``Complex.'' denotes the explanation complexity, i.e., the number of components required to fully explain model predictions, with $n$ representing the number of features.}
    \begin{center}
    \resizebox{\linewidth}{!}{
    \begin{tabular}{lccccccccccc}
        \toprule
        \bf \multirow{2.5}{*}{Model} & \bf \multirow{2.5}{*}{\makecell{Interpret-\\ability}} & \bf \multirow{2.5}{*}{Complex.} & \bf Housing & \bf \makecell{MIMIC-II} & \bf MIMIC-III & \bf Income & \bf Credit & \bf Year \\
        \cmidrule(lr){4-4}
        \cmidrule(lr){5-5}
        \cmidrule(lr){6-6}
        \cmidrule(lr){7-7}
        \cmidrule(lr){8-8}
        \cmidrule(lr){9-9}
         & & & RMSE $\downarrow$ & AUC $\uparrow$ & AUC $\uparrow$ & AUC $\uparrow$ & AUC $\uparrow$ & MSE $\downarrow$ \\
        \midrule
        \rowcolor[RGB]{234, 238, 234} \multicolumn{9}{c}{\textbf{Black-box Models}} \\
        \midrule
        MLP & No & N/A & \makecell{0.501 $\pm$ 0.006} & \makecell{0.835 $\pm$ 0.014} & \makecell{0.815 $\pm$ 0.009} & \makecell{0.914 $\pm$ 0.003} & \makecell{0.981 $\pm$ 0.007} & \makecell{{78.48} $\pm$ 0.56} \\
        XGBoost & No & N/A & \makecell{{0.443} $\pm$ 0.000} & \makecell{{0.844} $\pm$ 0.012} & \makecell{0.819 $\pm$ 0.004} & \makecell{{0.928} $\pm$ 0.003} & \makecell{0.978 $\pm$ 0.009} & \makecell{{78.53} $\pm$ 0.09} \\
        \midrule
        \rowcolor[RGB]{234, 238, 234} \multicolumn{9}{c}{\textbf{Interaction-explaining Models}} \\
        \midrule
        EB$^2$M & Yes & $O(n^2)$ & \makecell{{0.492} $\pm$ 0.000} & \makecell{{0.848} $\pm$ 0.012} & \makecell{0.821 $\pm$ 0.004} & \makecell{{0.928} $\pm$ 0.003} & \makecell{{0.982} $\pm$ 0.006} & \makecell{83.16 $\pm$ 0.01} \\
        NA$^2$M & Yes & $O(n^2)$ & \makecell{0.492 $\pm$ 0.008} & \makecell{0.843 $\pm$ 0.012} & \makecell{{0.825} $\pm$ 0.006} & \makecell{0.912 $\pm$ 0.003} & \makecell{{0.985} $\pm$ 0.007} & \makecell{79.80 $\pm$ 0.05} \\
        \midrule
        \rowcolor[RGB]{204, 255, 255} \multicolumn{9}{c}{\textbf{Feature-explaining Models}} \\
        \midrule
        EBM & Yes & $O(n)$ & \makecell{0.559 $\pm$ 0.000} & \makecell{0.835 $\pm$ 0.011} & \makecell{0.809 $\pm$ 0.004} & \makecell{{0.927} $\pm$ 0.003} & \makecell{0.974 $\pm$ 0.009} & \makecell{85.81 $\pm$ 0.11} \\
        NAM & Yes & $O(n)$ & \makecell{0.572 $\pm$ 0.005} &  \makecell{0.834 $\pm$ 0.013} & \makecell{0.813 $\pm$ 0.003} & \makecell{0.910 $\pm$ 0.003} & \makecell{0.977 $\pm$ 0.015} & \makecell{85.25 $\pm$ 0.01} \\
        NAE-D & Yes & $O(n)$ & \makecell{0.553 $\pm$ 0.001} & \makecell{0.830 $\pm$ 0.012} & \makecell{0.805 $\pm$ 0.006} & \makecell{{0.927} $\pm$ 0.003} & \makecell{0.977 $\pm$ 0.010} & \makecell{85.60 $\pm$ 0.04} \\
        \bottomrule
    \end{tabular}
    }
    \end{center}
     \label{tab:mixnam-d_acc}
\end{table*}

\begin{figure}[h!]
    \centering
    \includegraphics[width=1\linewidth]{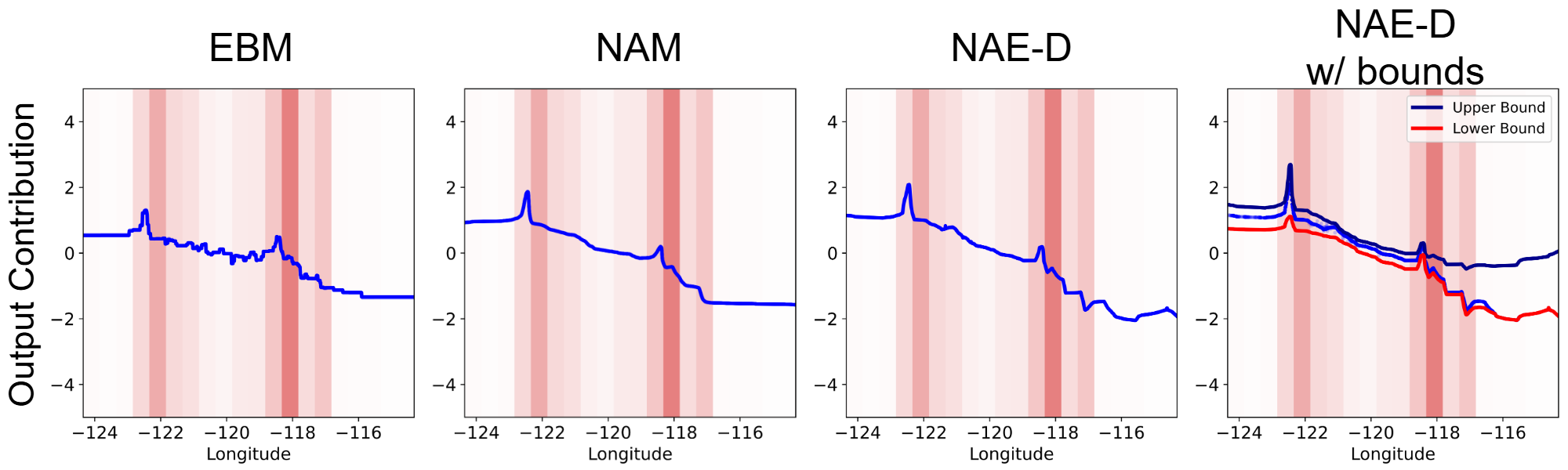}
    \caption{Visualization of the ``Longitude'' feature impact on house prices by NAE-D.}
    \label{fig:mixnam-d}
\end{figure}

Figure \ref{fig:mixnam-d} shows the shape plots generated by NAE-D compared to other baseline additive models. As discussed above, NAE-D is a strictly additive model where the effect of each feature on the final output is determined by the feature itself. Thus, we plot NAE-D in the same way as we did for EBM and NAM, directly showing the estimated contribution given by each feature value. Moreover, since NAE-D is trained under the general NAE framework with multiple experts learned for each feature, we plot another figure for NAE-D including the estimated upper and lower bounds based on the learned experts.

It can be observed from Figure \ref{fig:mixnam-d} that the patterns captured by NAE-D are similar to the ones by EBM and NAM. Beyond the point estimation provided by traditional additive models, NAE-D offers additional insights with its prediction bounds. In areas with low data density, the plot displays a wide gap between the upper and lower bounds, highlighting variability in predictions where less data is available. 

\section{NAE-E: NAE WITH EVENLY DISTRIBUTED EXPERT ACTIVATION} \label{sec:mixnam-e}

While the activation strategy detailed in Section \ref{sec:gating} facilitates continuous expert relevance estimation, we propose an adaptation that shifts this estimation to a discrete framework. This modification involves adjusting the weights to be evenly distributed across relevant experts, leading to a modified model termed NAE-E. To implement this, we use the masking vector outlined in Equation \eqref{eq:aggregation} and modify the relevance computation in Equation \eqref{eq:relevance} as follows:
\begin{equation} \label{eq:relevance_E}
    r_{ik} = \frac{\exp(\varphi_i[k] - \varphi_i^{*}[k] + M_i[k])}{\sum_{l=1}^K\exp(\varphi_i[l] - \varphi_i^{*}[l] + M_i[l])}.
\end{equation}
\(\varphi_i^{*}\) mirrors the values of \(\varphi_i\) but does not require gradient computations. The activation strategy implemented by Equation \eqref{eq:relevance_E} will result in a finite set of possible outcomes for each input feature value. Given a configuration of \(K\) experts and \(C\) activated ones per feature, the predicted output by NAE-E for any given feature will be one of the possible selections from \(\frac{K!}{C!(K-C)!}\) combinations.

\begin{table*}[h!] \small
    \centering
    \caption{Performance (RMSE\(\downarrow\)) and explanations of NAE-E with different variation penalties.}
    \begin{center}
    \begin{tabular}{ccccccc}
    \toprule
        \bf \(\lambda=\) 0 & \bf \(\lambda=\) 0.1 & \bf \(\lambda=\) 1 & \bf \(\lambda=\) 10 & \bf \(\lambda=\) 100 \\
    \midrule
        \rowcolor[RGB]{234, 238, 234} \multicolumn{5}{c}{RMSE (\(\downarrow\))} \\
        0.462 $\pm$ 0.002 & 0.460 $\pm$ 0.002 & 0.483 $\pm$ 0.001 & 0.551 $\pm$ 0.002 & 0.585 $\pm$ 0.003 \\
        \midrule
        \rowcolor[RGB]{234, 238, 234} \multicolumn{5}{c}{Shape Plot} \\
        \includegraphics[width=0.17\linewidth] {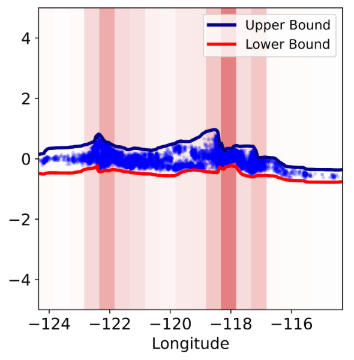} & \includegraphics[width=0.17\linewidth] {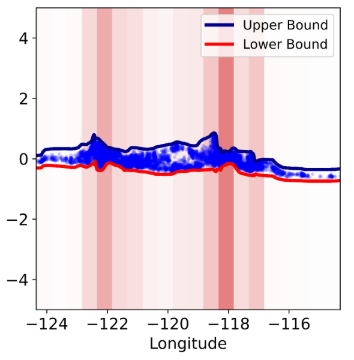} & \includegraphics[width=0.17\linewidth] {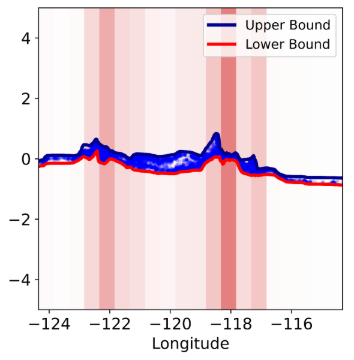} & \includegraphics[width=0.17\linewidth] {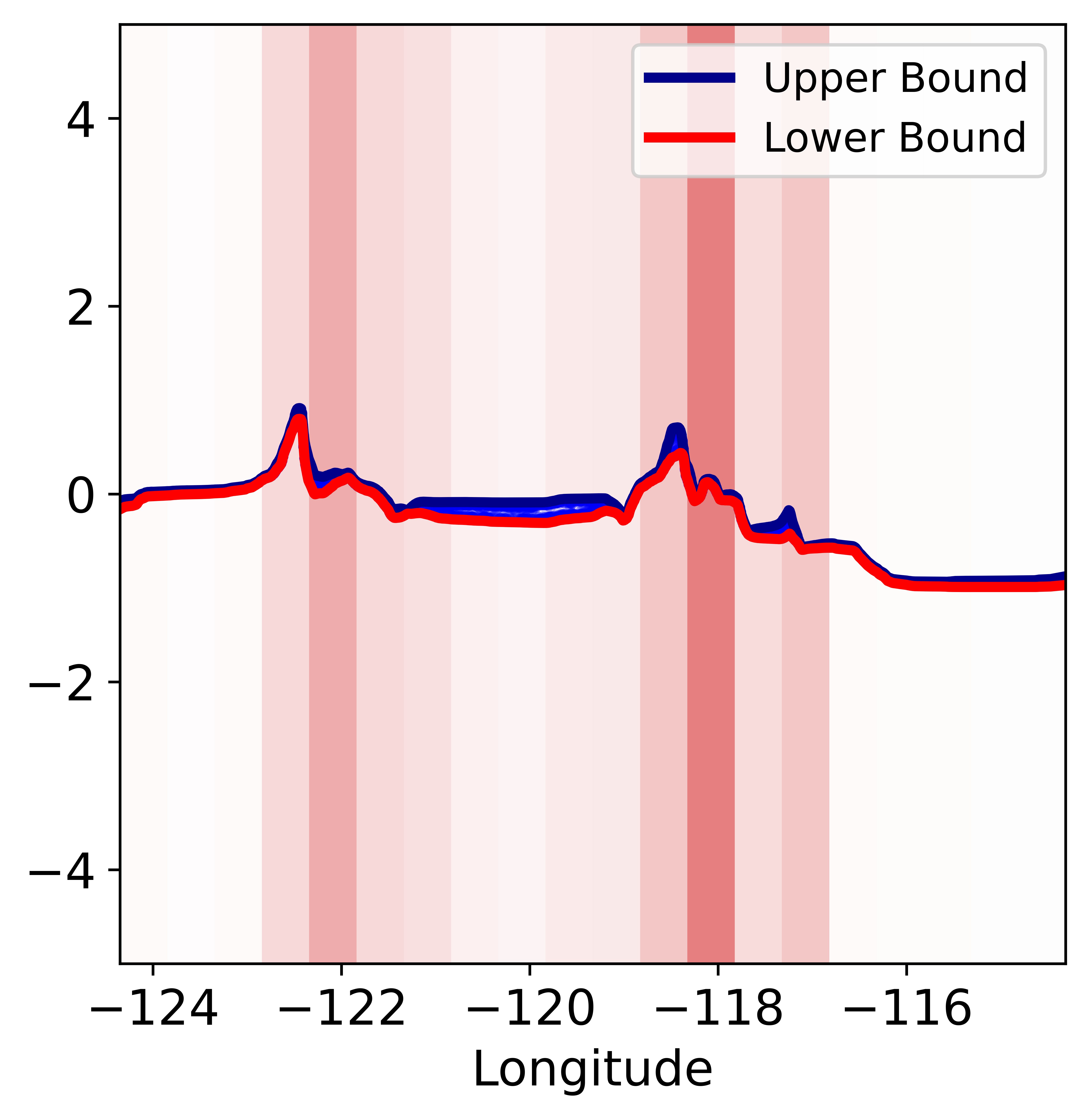} & \includegraphics[width=0.17\linewidth] {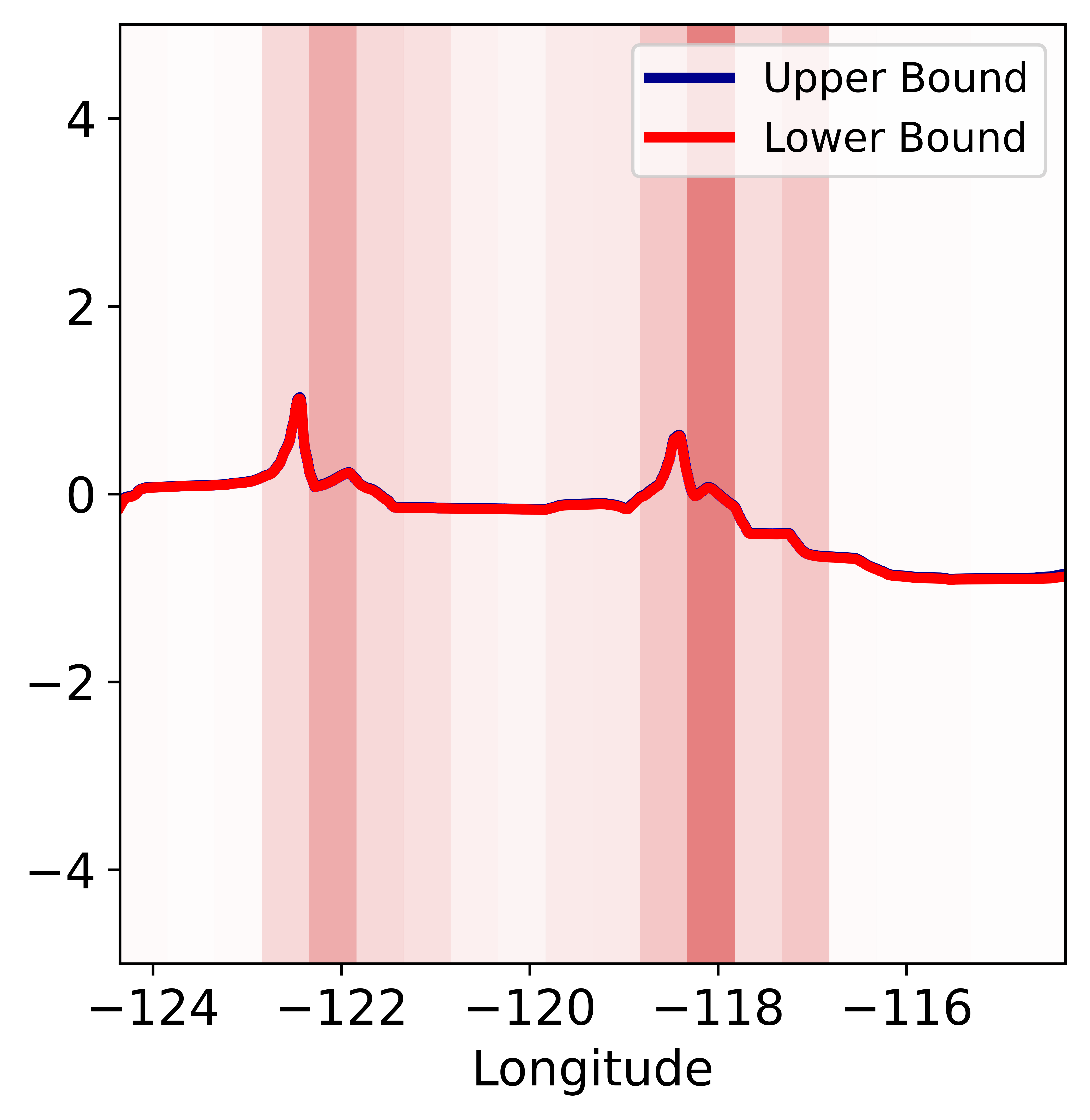} \\
    \bottomrule
    \end{tabular}
    \end{center}
    \label{tab:MixNAM-E}
\end{table*}

Table \ref{tab:MixNAM-E} shows the performance (RMSE) of NAE-E on the Housing dataset with different variation penalties, along with their corresponding shape plots of the ``Longitude'' feature. All the tested NAE-E models have a configuration of \(K=32\) and \(C=16\). In the plots, we illustrate the upper and lower bounds for each feature by selecting the maximum and minimum values from the \(\frac{K!}{C!(K-C)!}\) possible combinations. 

Compared with NAE (Table \ref{tab:additivity_tightness}), the performance of NAE-E is slightly worse due to the limitations imposed by its discrete range. However, even without any variation penalty, the estimated bounds by NAE-E still fit the actual score distribution tightly, accurately reflecting the prediction limits.

\section{PERFORMANCE OF NAE WITH THE SCALING OF EXPERTS} \label{sec:scaling}

As NAE employs a mixture of experts for data modeling, we explore how the number of total experts (\(K\)) and the number of activated experts (\(C\)) affect the overall model performance. Table \ref{tab:ablation} presents the performance of both NAE and NAE-E on the Housing dataset under various configurations. The results indicate that NAE reaches optimal performance at \(K=8\) and \(C=4\), and further increases in \(K\) or \(C\) do not enhance its performance. Such a finding is consistent with the role of experts in NAE, which primarily explore the variability in the prediction space to learn the upper and lower bounds, as the dynamic routing mechanism offers a continuous range within the bounds for prediction.
In contrast, NAE-E benefits from an increased number of \(K\) and \(C\), showing improved performance as these parameters grow. This difference shows how the continuous expert activation in NAE contrasts with the discrete one in NAE-E, highlighting the flexibility of our NAE framework to adapt to various scenarios through different implementations.

\begin{table}[h!] \small
    \centering
    \caption{Results for different numbers of activated experts ($C$) and all experts ($K$). The top three scores of each model are highlighted.}
    
\begin{center}
    %  NAE Table
    \begin{tabular}{ccccc}
        \toprule
        \bf \multirow{2.5}{*}{$C$ $\backslash$ $K$} & \multicolumn{4}{c}{\bf  NAE} \\
        \cmidrule(lr){2-5}
        & \textbf{2} & \textbf{4} & \textbf{8} & \textbf{16} \\
        \midrule
        \textbf{2} & 0.479 $\pm$ 0.003 & 0.459 $\pm$ 0.003 & 0.455 $\pm$ 0.004 & 0.460 $\pm$ 0.004 \\
        \textbf{4} & -- & \underline{0.451} $\pm$ 0.002 & \underline{0.447} $\pm$ 0.004 & 0.455 $\pm$ 0.004 \\
        \textbf{8} & -- & -- & \underline{0.449} $\pm$ 0.004 & 0.456 $\pm$ 0.003 \\
        \textbf{16} & -- & -- & -- & 0.456 $\pm$ 0.005 \\
        \specialrule{0.8pt}{0pt}{0pt}
    \end{tabular}
    %  NAE-E Table
    \begin{tabular}{ccccc}
        \bf \multirow{2.5}{*}{$C$ $\backslash$ $K$} & \multicolumn{4}{c}{\bf  NAE-E} \\
        \cmidrule(lr){2-5}
        & \textbf{16} & \textbf{32} & \textbf{64} & \textbf{128} \\
        \midrule
        \textbf{2}  & 0.486 $\pm$ 0.007 & 0.487 $\pm$ 0.005 & 0.486 $\pm$ 0.006 & 0.489 $\pm$ 0.005 \\
        \textbf{4}  & 0.469 $\pm$ 0.007 & 0.466 $\pm$ 0.005 & 0.468 $\pm$ 0.004 & 0.470 $\pm$ 0.004 \\
        \textbf{8}  & 0.470 $\pm$ 0.003 & 0.459 $\pm$ 0.004 & \underline{0.458} $\pm$ 0.002 & 0.458 $\pm$ 0.004 \\
        \textbf{16} & 0.587 $\pm$ 0.003 & 0.462 $\pm$ 0.002 & \underline{0.451} $\pm$ 0.002 & \underline{0.453} $\pm$ 0.002 \\
        \bottomrule
    \end{tabular}
\end{center}

    \label{tab:ablation}
\end{table}

To examine the interplay between the number of experts and the variation penalty, we conduct a grid search on the Housing dataset and report the RMSE scores in Table \ref{tab:k_lambda}. The results indicate that increasing \(K\) does not always lead to improved performance, and higher values of \(K\) require larger values of \(\lambda\) for regularization to prevent overfitting. This suggests that while more experts can capture more complex patterns, they also increase the risk of overfitting, necessitating stronger regularization to maintain generalization performance.

\begin{table}[h]
\centering
\caption{Effect of $\lambda$ for different values of $K$.}
\begin{center}
\begin{tabular}{ccccc}
\toprule
$K$ $\backslash$ $\lambda$ & 0 & 0.1 & 1 & 10 \\
\midrule
2  & 0.4805 & 0.4831 & 0.4989 & 0.5771 \\
4  & 0.4498 & 0.4498 & 0.4577 & 0.5068 \\
8  & 0.4507 & 0.4470 & 0.4460 & 0.4765 \\
16 & 0.4662 & 0.4528 & 0.4464 & 0.4509 \\
\bottomrule
\end{tabular}
\end{center}
\label{tab:k_lambda}
\end{table}

\section{DISCUSSION ON MODEL COMPLEXITY} \label{sec:complexity}
As NAE is a fundamental framework aiming to extend generalized additive models, our design prioritizes the performance and comprehensiveness of the overall framework over the optimization of model efficiency. On the basis of Neural Additive Models (NAMs), the neural-network-based Generalized Additive Models (GAMs), NAE introduces more parameters in the expert encoding and dynamic routing steps. To rigorously analyze the additional cost incurred by NAE compared to NAM, suppose we have \(n\) features, \(d\)-dimensional output for feature encoding, and \(K\) total experts for each feature. 
In our experiments, NAE is implemented with \(d=128\) and \(K=4\). For NAE-D discussed in Appendix \ref{sec:mixnam-d}, we have \(d=128\) with \(K=64\). The theoretical and empirical additional costs brought by NAE and its variant are presented in Table \ref{tab:complexity}, including the increase in parameter count and the additional memory required (assuming float32 storage). 

\begin{table}[h!] \small
\centering
\caption{Theoretical and empirical additional costs brought by  NAE compared to NAM.  NAE is implemented with \(d=128,K=4\).  NAE-D is implemented with \(d=128,K=64\). \(n\) denotes the number of input features.
}
\begin{center}
\begin{tabular}{l|c|cc|cc}
\toprule
& \bf \multirow{3}{*}{\bf \(n\)} & \multicolumn{2}{c|}{\bf { NAE}} & \multicolumn{2}{c}{\bf { NAE-D}} \\
\cmidrule(lr){3-4} \cmidrule(lr){5-6}
 &  & \bf \makecell{Parameter\\Count} & \bf \makecell{Memory\\Usage} & \bf \makecell{Parameter\\Count} & \bf \makecell{Memory\\Usage} \\
\midrule
 Housing & 8 & 37k & 144k & 132k & 516k \\
 MIMIC-II & 17 & 157k & 613k & 281k & 1.1M \\
 MIMIC-III & 57 & 1.7M & 6.5M & 941k & 3.59M \\
 Income & 14 & 108k & 420k & 231k & 903k \\
 Credit & 30 & 476k & 1.8M & 495k & 1.89M \\
 Year & 90 & 4.2M & 16.0M & 1.5M & 5.67M \\
\midrule
Theoretical & -- & \multicolumn{2}{c|}{\(nK[(n+1)d+2]\)} & \multicolumn{2}{c}{\(nK(2d+2)\)} \\
\bottomrule
\end{tabular}
\label{tab:complexity}
\end{center}
\end{table}

The results indicate that the additional cost associated with NAE-D scales linearly with the number of features, whereas the cost for NAE includes a squared term due to its more complex routing mechanism. Specifically, NAE utilizes a full \(n\times n\) block matrix to encode feature interactions, which is simplified as a block-diagonal matrix in NAE-D. Despite this increased complexity, the additional memory usage remains manageable with current computing resources.
Future efforts could be made to sparsify the routing matrix in NAE, which could potentially reduce the additional cost while retaining the model performance.

Additionally, we conducted experiments to measure actual training time (in seconds) / memory usage (in MB) across varying numbers of features ($n$) and experts ($K$) for NAE. Table \ref{tab:complexity_exp} summarizes the results with 5000 training samples and 100 training epochs, using the encoder architecture from Table \ref{tab:hyerparams} with four layers, each containing 128 neurons.

\begin{table}[h!] \small
\centering
\caption{Training time / memory usage of NAE with different $K$ and $n$ values.}
\begin{center}
\begin{tabular}{lccccc}
\toprule
& $n = $ 32 & $n = $ 64 & $n = $ 128 & $n = $ 256 & $n = $ 512 \\
\midrule
$K = $ 4 & 39.7s / 2.6MB & 39.8s / 7.2MB & 42.8s / 22.3MB & 64.4s / 76.6MB & 174.0s / 281.2MB \\
$K = $ 8 & 40.4s / 3.6MB & 42.6s / 11.2MB & 47.6s / 38.4MB & 102.3s / 140.8MB & 238.0s / 537.7MB \\
$K = $ 16 & 40.8s / 5.7MB & 42.3s / 19.3MB & 82.9s / 70.7MB & 150.1s / 269.4MB & 346.7s / 1050.7MB \\
$K = $ 32 & 42.9s / 9.8MB & 89.6s / 35.6MB & 111.7s / 135.2MB & 277.5s / 526.4MB & 559.0s / 2076.8MB \\
\bottomrule
\end{tabular}
\end{center}
\label{tab:complexity_exp}
\end{table}

The results validate the theoretical analysis in Table \ref{tab:complexity}, showing that NAE's training time and memory usage increase with both the number of features and experts.
For comparison, we also measured end-to-end training times and memory usage for NA$^2$M, EB$^2$M, NAM, and EBM, comparing them with NAE when \(K=4\), which is the setting used in our main experiments.
As the results in Table \ref{tab:complexity_exp2} show, while NAE preserves the capability to model interactions, it presents a computational cost that is lower than NA$^2$M and EB$^2$M, especially when the number of features is large.

\begin{table}[h!] \small
\centering
\caption{Training time and memory usage of NA$^2$M, EB$^2$M, NAM, and EBM for varying numbers of features ($n$). ``--'' indicates experiments that exceeded one hour to complete.}
\begin{center}
\resizebox{1.0\textwidth}{!}{ 
\begin{tabular}{lccccccc}
\toprule
& $n = $ 32 & $n = $ 64 & $n = $ 128 & $n = $ 256 & $n = $ 512 \\
\midrule
NA$^2$M & 70.9s / 24.2MB & 290.7s / 99.0MB & 1146.1s / 402.9MB & -- & -- \\
EB$^2$M & 4.5s / 2.2MB & 10.5s / 6.6MB & 34.4s / 22.5MB & 155.4s / 82.2MB & 1004.3s / 313.2MB \\
NAE ($K = $ 4) & 39.7s / 2.6MB & 39.8s / 7.2MB & 42.8s / 22.3MB & 64.4s / 76.6MB & 174.0s / 281.2MB \\
NAM & 6.3s / 1.6MB & 10.4s / 3.1MB & 17.2s / 6.3MB & 33.0s / 12.9MB & 66.0s / 26.8MB \\
EBM & 5.7s / 1.1MB & 9.4s / 2.1MB & 18.0s / 3.7MB & 37.6s / 4.5MB & 95.8s / 5.3MB \\
\bottomrule
\end{tabular}
}
\end{center}
\label{tab:complexity_exp2}
\end{table}

\section{MORE VISUALIZATION RESULTS OF NAE ON REAL-WORLD DATA} \label{appendix:more_visual}

Figures \ref{fig:housing_all} - \ref{fig:credit_all} show the complete visualization results of feature explanations by NAE on different real-world datasets. Here we present results for datasets with no more than 50 features, including Housing, MIMIC-II, Income, and Credit.

\begin{figure*}[h!]
    \centering
    \includegraphics[width=1\linewidth]{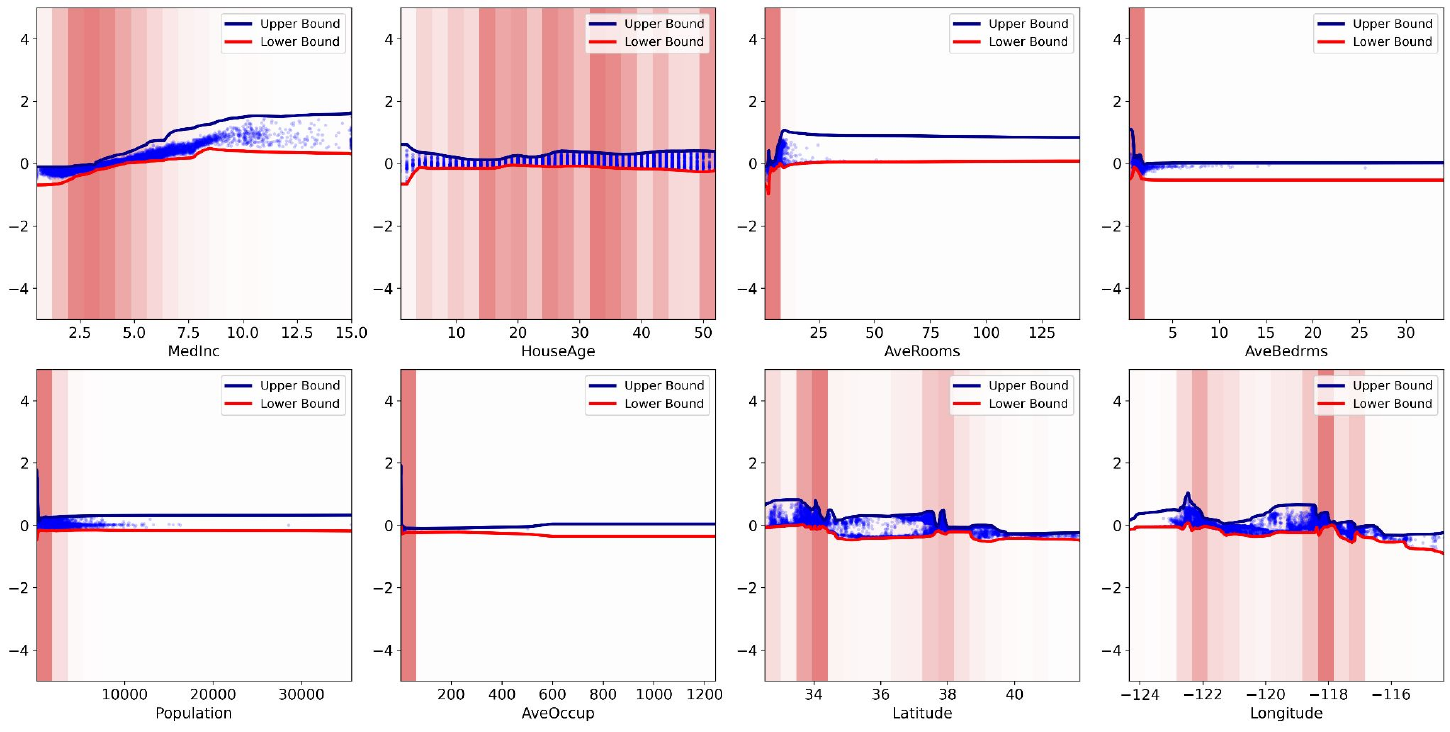}
    \caption{Visualization of feature explanations by NAE on the Housing dataset.}
    \label{fig:housing_all}
\end{figure*}

\begin{figure*}[h!]
    \centering
    \includegraphics[width=0.9\linewidth]{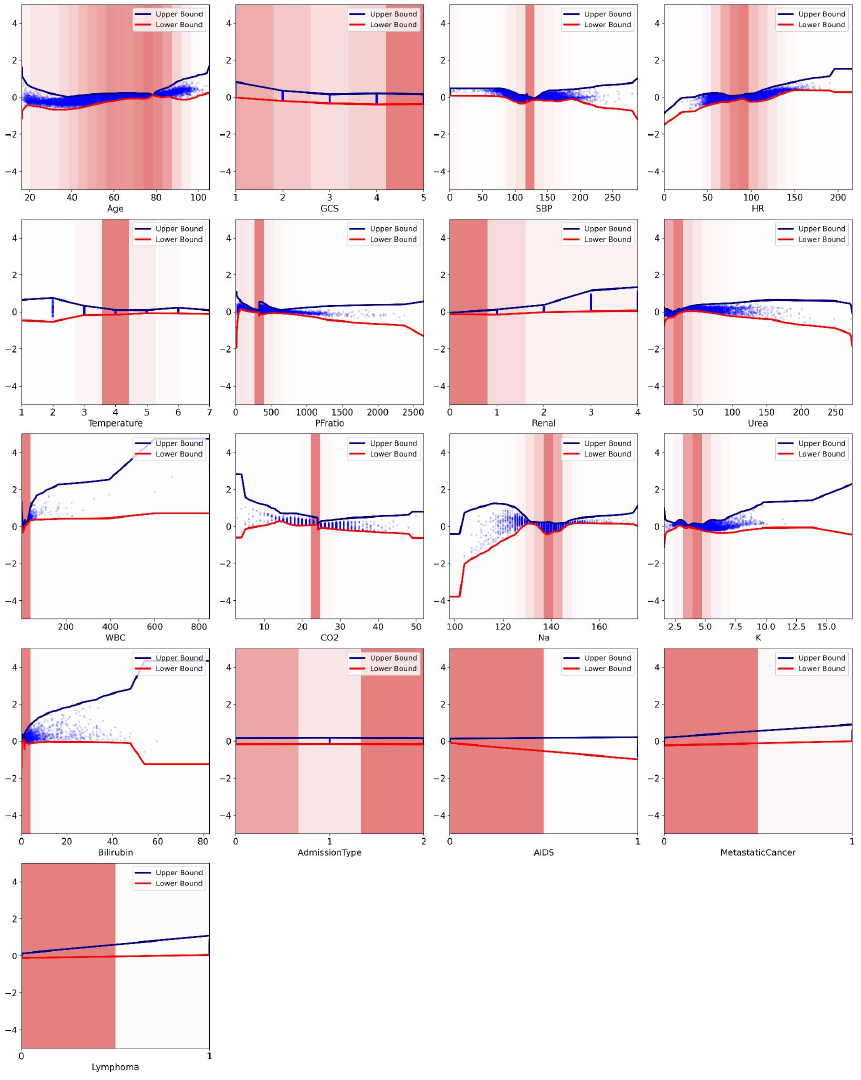}
    \caption{Visualization of feature explanations by NAE on the MIMIC-II dataset.}
    \label{fig:mimic2_all}
\end{figure*}

\begin{figure*}[h!]
    \centering
    \includegraphics[width=0.85\linewidth]{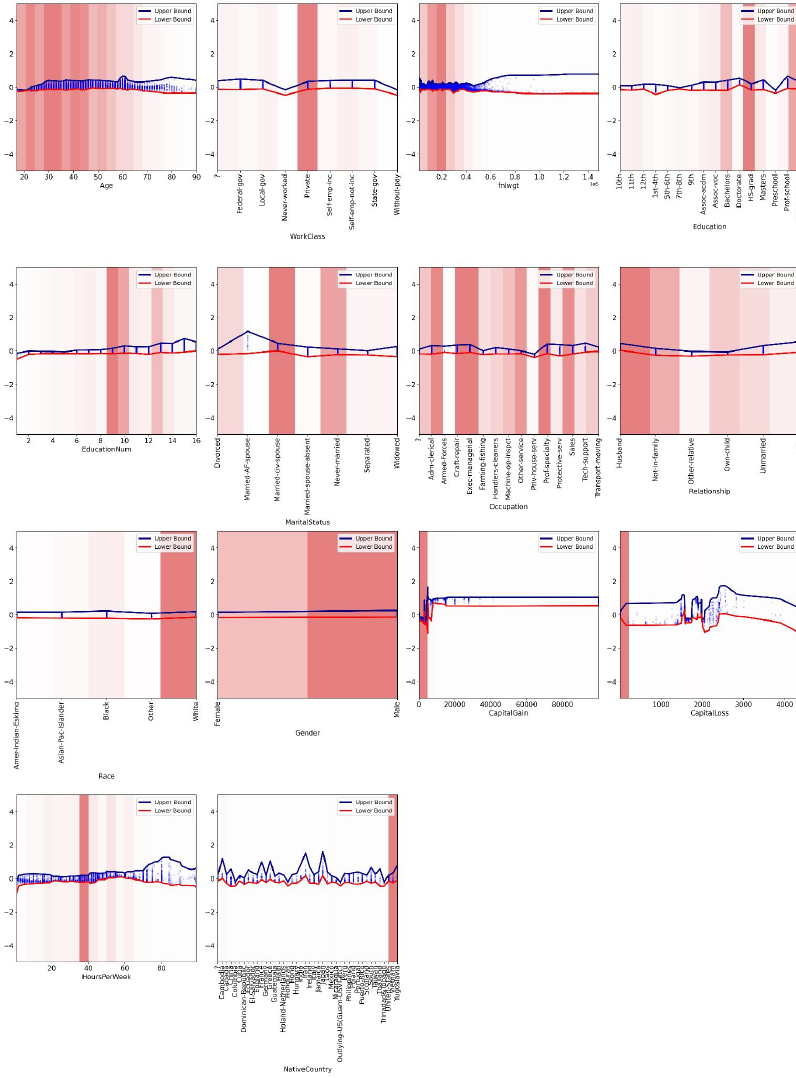}
    \caption{Visualization of feature explanations by NAE on the Income dataset.}
    \label{fig:income_all}
\end{figure*}

\begin{figure*}[h!]
    \centering
    \includegraphics[width=0.95\linewidth]{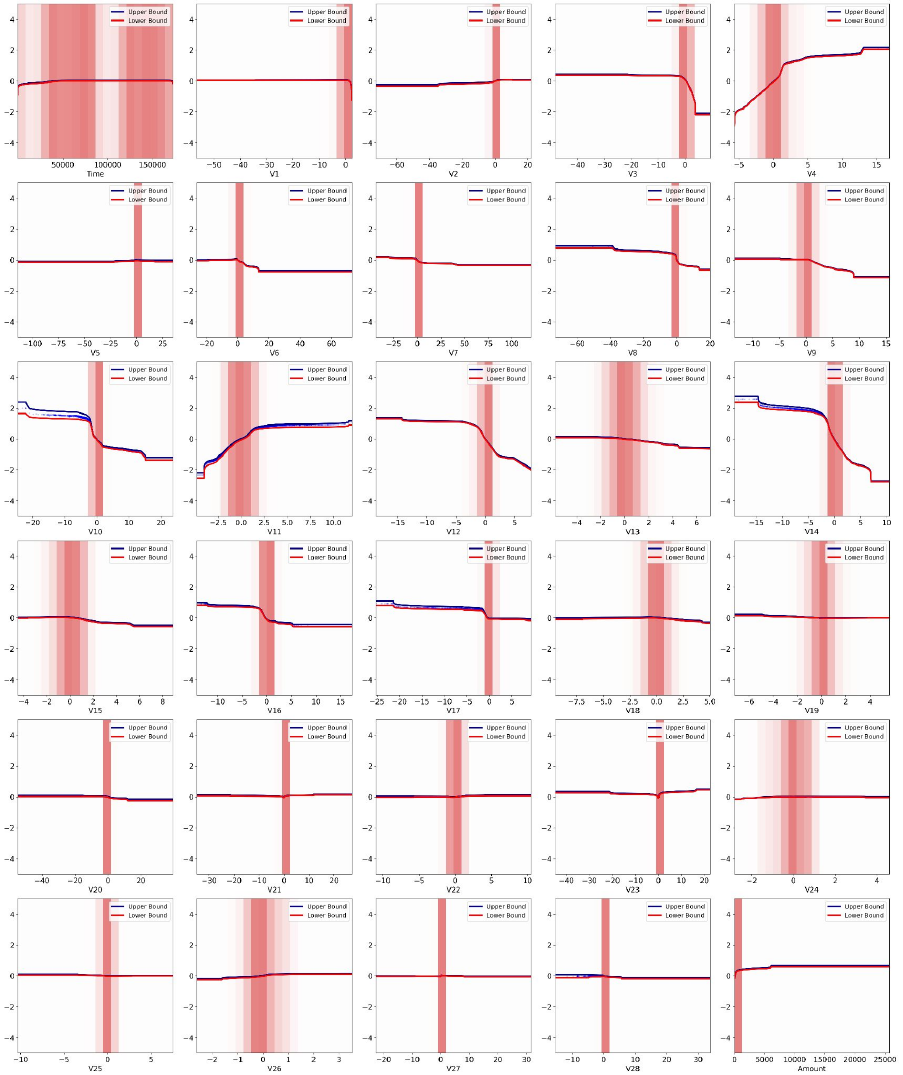}
    \caption{Visualization of feature explanations by NAE on the Credit dataset.}
    \label{fig:credit_all}
\end{figure*}

\end{document}